\documentclass{article}

\usepackage{tabularx}

\usepackage[final]{neurips_2024}
\usepackage{tabularx, booktabs, array}

\usepackage{amsmath}
\usepackage{amssymb}
\usepackage{amsthm}
\newtheorem{theorem}{Theorem}[section]
\newtheorem{definition}{Definition}[section]
\newtheorem{lemma}{Lemma}[section]
\newtheorem{remark}{Remark}[section]

\usepackage[utf8]{inputenc} 
\usepackage[T1]{fontenc}    
\usepackage{hyperref}       
\usepackage{url}            
\usepackage{booktabs}       
\usepackage{amsfonts}       
\usepackage{nicefrac}       
\usepackage{microtype}      
\usepackage{xcolor}         
\usepackage{tikz}
\usepackage{relsize}

\usepackage{algorithm}
\usepackage{algpseudocode}
\usepackage{hyperref}
\hypersetup{colorlinks,
linkcolor=blue,
citecolor=blue,
urlcolor=magenta,
linktocpage,
plainpages=false}
\usepackage{xfrac}

\newcommand{\A}{\mathcal{A}}

\newcommand{\G}{\mathcal{G}}

\newcommand{\F}{\mathcal{F}}

\newcommand{\MMM}{\mathcal{M}}
\newcommand{\BBB}{\mathcal{B}}
\newcommand{\EEE}{\mathcal{E}}

\newcommand{\GGG}{\mathcal{G}}

\newcommand{\tsigma}{\tilde{\sigma}}

\newcommand{\PS}{\mathcal{PS}}

\newcommand{\ba}{{\bf a}}

\newcommand{\bb}{{\bf b}}
\newcommand{\bc}{{\bf c}}
\newcommand{\bd}{{\bf d}}

\newcommand{\bg}{{\bf g}}
\newcommand{\bz}{{\bf z}}

\newcommand{\by}{{\bf y}}

\newcommand{\bw}{{\bf w}}
\newcommand{\bx}{{\bf x}}
\newcommand{\bh}{{\bf h}}
\newcommand{\real}{\mathbb{R}}
\newcommand{\D}{{\mathcal{D}}}
\newcommand{\K}{{\mathcal{K}}}

\newcommand{\B}{\mathcal{B}}
\newcommand{\supp}{\textbf{Support}\{\D \}}
\newcommand{\sigmal}{\sigma_L}

\newcommand{\E}{\mathbb{E}}
\usepackage{amsmath}
\usepackage{enumitem}

\setlist[itemize]{left=1pt}

\title{Weight for Robustness: A Comprehensive Approach towards Optimal Fault-Tolerant Asynchronous ML
}

\author{%
  Tehila Dahan \\
  Department of Electrical Engineering\\
  Technion\\
  Haifa, Israel \\
\texttt{t.dahan@campus.technion.ac.il} \\
  \And
 Kfir Y. Levy \\
  Department of Electrical Engineering\\
  Technion\\
  Haifa, Israel \\
\texttt{kfirylevy@technion.ac.il} \\
}

\begin{document}

\maketitle

\begin{abstract}
We address the challenges of Byzantine-robust training in asynchronous distributed machine learning systems, aiming to enhance efficiency amid massive parallelization and heterogeneous computing resources. Asynchronous systems, marked by independently operating workers and intermittent updates, uniquely struggle with maintaining integrity against Byzantine failures, which encompass malicious or erroneous actions that disrupt learning. The inherent delays in such settings not only introduce additional bias to the system but also obscure the disruptions caused by Byzantine faults. To tackle these issues, we adapt the Byzantine framework to asynchronous dynamics by introducing a novel weighted robust aggregation framework. This allows for the extension of robust aggregators and a recent meta-aggregator to their weighted versions, mitigating the effects of delayed updates. By further incorporating a recent variance-reduction technique, we achieve an optimal convergence rate for the first time in an asynchronous Byzantine environment. Our methodology is rigorously validated through empirical and theoretical analysis, demonstrating its effectiveness in enhancing fault tolerance and optimizing performance in asynchronous ML systems.
\end{abstract}

\section{Introduction} 
In recent years, there has been significant growth in the development of large-scale machine learning (ML) models and the volume of data they require \citep{zhao2023survey}. To efficiently accelerate large-scale training processes, Distributed ML has emerged as a crucial approach that can be categorized into synchronous and asynchronous paradigms. In synchronous learning, workers update the model simultaneously using the average of their outputs, similar to the Minibatch approach \citep{dekel2012optimal}. Asynchronous learning, however, allows workers to operate independently, sending updates as they are ready without waiting for others \citep{arjevani2020tight}.  This prevents slow workers from hindering the process, making it especially practical 
as the number of workers increases.

A major challenge of distributed ML is fault-tolerance, and  
Byzantine ML \citep{ alistarh2018byzantine, lamport2019byzantine,guerraoui2023byzantine} is a powerful framework for tackling this aspect. Byzantine ML captures a broad spectrum of failures within distributed environments, including random malfunctions or even malicious workers aiming to disrupt the training process. This makes Byzantine ML  widely applicable across various domains to ensure robust performance.

Addressing the Byzantine problem in synchronous distributed learning is well-established \citep{karimireddy2020byzantine, karimireddy2021learning, allouah2023fixing, farhadkhani2022byzantine, alistarh2018byzantine, dahanlevy2024}. Two primary ingredients were found to be crucial towards tackling Byzantine ML in synchronous settings: \textbf{(i)} \emph{Robust Aggregators} \citep{yin2018byzantine, blanchard2017machine, chen2017distributed}: such aggregators combine the gradient estimates sent by the workers to a single estimate while filtering out the outliers which may hinder the training process.
While the use of robust aggregators is crucial, it was found to be insufficient, and an additional ingredient of 
 \textbf{(ii)} \emph{learning from history} was shown to be vital in mitigating Byzantine faults \citep{karimireddy2021learning}.
And, the performance of robust aggregators was systematically explored within a powerful generic framework  \citep{karimireddy2020byzantine, karimireddy2021learning, allouah2023fixing, farhadkhani2022byzantine, dahanlevy2024}. Moreover,  due to the diversity of Byzantine scenarios \citep{xie2020fall, allen2020byzantine, baruch2019little}, it was found that relying on a single aggregator is insufficient, making the variety of robust aggregators essential. Unfortunately, many existing aggregators have sub-optimal performance. This drawback was elegantly resolved by the design of meta-aggregators \citep{karimireddy2020byzantine, allouah2023fixing, dahanlevy2024}, that enable to boost the performance of baseline aggregators. Unfortunately, in the asynchronous case, the use of robust aggregators is not straightforward, as updates are typically applied individually per-worker, rather than averaging outputs from all workers at once \citep{arjevani2020tight}. 

Despite its advantages, asynchronous distributed learning presents unique challenges, particularly when dealing with Byzantine faults. The delays inherent in asynchronous settings introduce additional bias to the system and obscure the disruptions caused by Byzantine faults. In fact, in contrast to the synchronous Byzantine setting, all existing approaches towards the asynchronous Byzantine case do not ensure a generalization error (excess loss) that diminishes with the number of honest data-samples and updates. This applies to works for both convex~\citep{fang2022aflguard} as well as non-convex scenarios~\citep{xie2020zeno++, yang2023buffered}; as well as to works that further assume the availability of a \emph{trusted dataset} possessed by the central-server~\citep{xie2020zeno++, fang2022aflguard}. Furthermore, the performance guarantees of all existing approaches towards that setting include an explicit dependence on the dimension of the problem --- a drawback that does not exist for SOTA synchronous Byzantine approaches. 

\textbf{Contributions.} We explore the asynchronous Byzantine setting under the fundamental framework of Stochastic Convex Optimization (SCO)~\citep{hazan2016introduction}. 
Our work is the first to achieve a convergence rate that diminishes with the number of honest data samples and updates and does not explicitly depend on the problem's dimension. In the absence of Byzantine workers, our rate matches the optimal performance of Byzantine-free asynchronous settings. This stands in contrast to previous efforts on Byzantine, which did not attain diminishing rates or dimensionality independence, even without Byzantine workers. We also show the effectiveness of our approaches in practice. 
Our contributions:
\begin{itemize}[leftmargin=1em]
    \item We quantify the difficulty in asynchronous scenarios by considering the \emph{number of Byzantine updates}, which is more natural than the standard measure of \emph{number of Byzantine workers}.
 \item We identify the need to utilize weighted aggregators rather than standard ones in favor of asynchronous Byzantine problems. Towards doing so, we extend the robust aggregation framework to allow and include weights and develop appropriate (weighted) rules and a meta-aggregator.
\item \textbf{Achieving Optimal Convergence}: We incorporate our weighted robust framework with a recent double momentum mechanism, leveraging its unique features to achieve an optimal convergence rate for the first time in asynchronous Byzantine ML.
\end{itemize}

\paragraph{Related Work.} A long line of studies has explored the synchronous Byzantine setting (see e.g.,~
\cite{alistarh2018byzantine,karimireddy2020byzantine, karimireddy2021learning, allouah2023fixing, farhadkhani2022byzantine, allen2020byzantine, el2021distributed, dahanlevy2024}). \cite{alistarh2018byzantine,karimireddy2021learning} demonstrated that historical information is crucial for optimal performance in Byzantine scenarios; and  \cite{karimireddy2021learning} introduced the idea of combining generic aggregation rules with standard momentum, using a parameter of $\Theta(1/\sqrt{T})$, which effectively incorporates $\Theta(\sqrt{T})$ iterations of historical gradients. Additionally, \cite{dahanlevy2024} showed that a double momentum approach is effective by taking a momentum parameter of \(1/T\), capturing the entire gradient history.

Robust aggregators such as Coordinate-wise Trimmed Mean (CWTM) \citep{yin2018byzantine}, Krum \citep{blanchard2017machine}, Geometric Median (GM) \citep{chen2017distributed}, CWMed \citep{yin2018byzantine}, and Minimum Diameter Averaging \citep{guerraoui2018hidden} have also proven to be highly beneficial in synchronous settings and have been evaluated within robust frameworks \citep{allouah2023fixing, karimireddy2020byzantine, farhadkhani2022byzantine, dahanlevy2024}.
However, not all robust aggregators achieve optimal performance, leading to the development of meta-aggregators \citep{karimireddy2020byzantine, allouah2023fixing, dahanlevy2024} to enhance their effectiveness. While standard aggregation works well in synchronous settings, where outputs are averaged across all workers, it is less suitable for asynchronous settings, where updates are processed individually as they arrive \citep{ arjevani2020tight}.

To adapt these approaches to asynchronous settings, \cite{yang2023buffered} devised BASGDm, an extension of BASGD \citep{yang2021basgd}, that groups worker momentums into buckets that are then aggregated using a robust aggregator. Other methods, like Zeno++ \citep{xie2020zeno++} and AFLGuard \citep{fang2022aflguard}, rely on a trusted dataset on the central server, which hinders their practicality. Kardam \citep{damaskinos2018asynchronous} uses the Lipschitzness of gradients to filter out outliers. Unfortunately, none of these approaches ensure a generalization error (excess loss) that diminishes with the number of honest data-samples and updates, and suffers from an explicit dependence on the problem's dimension. And this applies even in the absence of Byzantine faults.

Asynchronous Byzantine ML faces unique challenges as inherent delays add bias that obscures Byzantine disruptions. To mitigate this delay-bias in asynchronous, non-Byzantine scenarios, \cite{cohen2021asynchronous, aviv2021asynchronous} propose methods to keep model weights relatively close during iterations. Other approaches \citep{ stich2019error, arjevani2020tight,mishchenko2022asynchronous} suggest adjusting the step size proportionally to the delay. These strategies have proven useful in reducing the negative impact of delays, and achieve optimal performance.

Our work extends several concepts from \cite{dahanlevy2024} to the asynchronous scenario.
We devise a novel generalization of their Centered Trimmed Meta Aggregator (CTMA) towards weighted meta-aggregation, making it amenable to asynchronous scenarios.
In the spirit of \cite{dahanlevy2024}, we also adopt a recent variance reduction technique called \(\mu^2\)-SGD \citep{levy2023mu}. Nevertheless, while  \cite{dahanlevy2024} used this technique in a straightforward manner, we found it crucial to appropriately incorporate individual per-worker weights to overcome the challenge of asynchronicity in Byzantine ML. 

\section{Setting}
Our discussion focuses on the minimization of a smooth convex objective $f:\K \rightarrow \real$:
\begin{equation*}
    f(\bx):=\E_{\bz\sim\D}[f(\bx;\bz)]~,
\end{equation*}
where $\K\subseteq\real^d$ is a compact convex set and $\D$ denotes an unknown distribution from which we can draw i.i.d samples $\{\bz_t\sim\D\}_t$. Our work considers first-order methods that iteratively utilize gradient information to approach an optimal point. Such methods output a solution $\bx_T$, which is evaluated by the expected excess loss:
\begin{equation*}
    \textrm{ExcessLoss}:=\E[f(\bx_T) - f(\bx^*)]~,
\end{equation*}
where $\bx^*$ is a solution that minimizes $f$ over $\K$ and $\bx_T\in\K$ approximates this optimal solution.

\textbf{Asynchronous Training.} We explore these methods within a distributed environment involving multiple workers.
Our discussion focuses on a \emph{centralized} distributed framework characterized by a central Parameter Server ($\mathcal{PS}$) that may communicate with $m$ workers. Each of these workers may draw i.i.d. samples \( \bz \sim \D \); and based on these samples, compute unbiased gradient estimate $\bg\in\real^d$ at a point $\bx\in \K$. Concretely, a worker may compute $\bg:=\nabla f(\bx; \bz)$; implying that  $\E[\bg|\bx]=\nabla f(\bx)$. Specifically, our main focus is on \emph{Asynchronous} systems, where the $\mathcal{PS}$ does not wait to receive the stochastic gradient computations from all machines; instead, it updates its model whenever a worker completes a (stochastic) gradient computation. That worker then proceeds to compute a gradient estimate for the updated model, while the other workers continue to compute gradients based on `stale' models. This staleness leads to the use of staled (and therefore biased) gradient estimates, which is a major challenge in designing and analyzing asynchronous training methods.

\textbf{Asynchronous Byzantine Framework.} We assume that an unknown subset of the $m$ workers are \emph{Byzantine}, implying that these workers may transmit arbitrary or malicious information during the training process, and these "Byzantine" workers may even collaborate to disrupt the training.
We assume that the fraction of updates that arrive from Byzantine workers during the asynchronous training process is bounded and strictly smaller than $\nicefrac{1}{2}$ and denote this fraction by $\lambda$. 
\begin{remark}[Fraction of Byzantine Updates vs. Byzantine Workers]
In both synchronous and asynchronous settings, it is common to consider a bound on the \textbf{fraction of Byzantine workers} (up to \nicefrac{1}{2}) \citep{allouah2023fixing, farhadkhani2022byzantine, karimireddy2020byzantine, karimireddy2021learning, yang2023buffered, yang2021basgd, damaskinos2018asynchronous}. 
In synchronous scenarios this is meaningful since the server 
equally treats the information from all workers; which is done by equally
averaging gradients of all workers in each iteration in a mini-batch fashion \citep{dekel2012optimal}.
Conversely, in asynchronous scenarios, faster workers contribute to more updates than slower workers, leading to an unequal influence on the training process. In such scenarios, the fraction of Byzantine workers is less relevant; and it is therefore much more natural to consider the \textbf{fraction of Byzantine updates}. Interestingly, our definition aligns with the standard one (for the synchronous case), which considers the number of Byzantine workers.
\end{remark}

\textbf{Notation.}  For each worker \(i \in [m]\) and iteration \(t\), \(s_t^{(i)}\) represents the total number of updates by worker \(i\) up to \(t\), and \(\tau_t^{(i)}\) is the delay compared to the current model. \(t^{(i)}\) is the last update before \(t\), making \(\tau_t^{(i)}\) the time since the second last update (Figure~\ref{fig:delay}). \(\tau_t\) denotes the delay for the worker arriving at iteration \(t\), i.e., if worker $j$ arrives at iteration $t$ then \(\tau_t=\tau_t^{(j)}\). 

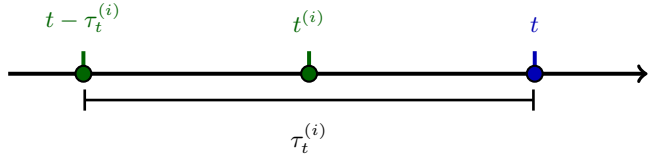
\begin{figure}[h]
    \centering
    \begin{tikzpicture}[scale=1]

        \draw[->, line width=1.5pt, color=black] (0,0) -- (8.5,0) node[right] {};

        \coordinate (T) at (7,0);
        \coordinate (Ti) at (4,0);
        \coordinate (Ttau) at (1,0);

        \draw[line width=1.5pt, color={rgb,255:red,0; green,0; blue,180}] (T) -- ++(0,0.3) node[above, font=\small, yshift=3pt] {$t$}; 
        \draw[line width=1.5pt, color={rgb,255:red,0; green,100; blue,0}] (Ti) -- ++(0,0.3) node[above, font=\small, yshift=3pt] {$t^{(i)}$}; 
        \draw[line width=1.5pt, color={rgb,255:red,0; green,100; blue,0}] (Ttau) -- ++(0,0.3) node[above, font=\small, yshift=3pt] {$t - \tau^{(i)}_t$}; 

        \filldraw[fill={rgb,255:red,0; green,0; blue,180}, draw=black, thick] (T) circle (3pt); 
        \filldraw[fill={rgb,255:red,0; green,100; blue,0}, draw=black, thick] (Ti) circle (3pt); 
        \filldraw[fill={rgb,255:red,0; green,100; blue,0}, draw=black, thick] (Ttau) circle (3pt); 

        \draw[|-|, line width=1.0pt, color=black] (Ttau |- 0,-0.35) -- (T |- 0,-0.35)
            node[midway, below=5pt, font=\small] {$\tau_{t}^{(i)}$};
    \end{tikzpicture}
    \caption{Illustration of the delay interval \(\tau^{(i)}_t\) for worker \(i\) at iteration \(t\), marking \(t\) (current iteration), \(t^{(i)}\) (most recent update from worker \(i\)), and \(t - \tau^{(i)}_t\) (previous update from worker \(i\)).}
    \label{fig:delay}
\end{figure}

    For a given time (iteration) \(t\), let \(t^{(i)}\) be the last iteration when worker \(i\) made an update. We denote $\bd_t^{(i)}:=\bd_{t^{(i)}}$, $\bg_t^{(i)}:=\bg_{t^{(i)}}$, $\Tilde{\bg}_t^{(i)}:=\Tilde{\bg}_{t^{(i)}}$, and $\bx_t^{(i)}=\bx_{t^{(i)}}$, where the latter are individual vectors that we will later define for any worker $i$. Throughout, \(\lVert\cdot\rVert\) represents the \(L_2\)-norm. For any natural \(N\), \(\left[N\right]=\left\{1,\ldots,N\right\}\). We use the compressed sum notation \(\alpha_{1:t}= \sum_{k=1}^{t}\alpha_k\). For every \(\bx \in \mathbb{R}^d\), the orthogonal projection of \(\bx\) onto a set \(\K\) is denoted by \(\Pi_\K(\bx)= \arg\min_{\by \in \K} \|\by - \bx\|^2\). We denote \(\BBB\) and \(\GGG\) as the subsets of Byzantine workers and honest workers, respectively, such that \(|m| = |\GGG| + |\BBB|\).

\paragraph{Assumptions.} We use the following conventional assumptions: 
\begin{flalign}
\label{eq:bounded_diameter}
\text{\textbf{Bounded Diameter}: we assume there exists } D>0 \text{ such that } \max_{\bx,\by\in\K}\|\bx-\by\|\leq D. &&
\end{flalign}
\textbf{Bounded Variance}:~there exists $\sigma>0$ such that $\forall \bx\in\K$, $\bz\in\supp$,
\begin{align}
\label{eq:bounded-variance}
    \E\|\nabla f(\bx;\bz) - \nabla f(\bx)\|^2\leq \sigma^2~.
\end{align}
\textbf{Expectation over Smooth Functions}:~we assume that $f(\cdot)$ is an expectation of smooth functions, i.e.~ $\forall \bx,\by\in\K~, \bz\in\supp $ there exist $L>0$ such that,
\begin{align} \label{eq:Main}
\|\nabla f(\bx;\bz) - \nabla f(\by;\bz)\| \leq L\|\bx-\by\|~,
\end{align} 
The above assumption also implies that the expected loss $f(\cdot)$ is $L$ smooth. \\
\textbf{Bounded Smoothness Variance} \citep{levy2023mu}:~
in Appendix \ref{sec:sigmal} we show that Eq.~\eqref{eq:Main} implies that, $\forall \bx,\by\in\K$, $\bz\in\supp$ there exists $\sigmal^2 \in[0,L^2]$ such,
\begin{align} 
\begin{split}
    \E\left\|(\nabla f(\bx;\bz)-\nabla f(\bx)) - (\nabla f(\by;\bz)-\nabla f(\by))\right\|^2 \leq \sigmal^2 \|\bx-\by\|^2 
\end{split}
\label{eq:sigmal}
\end{align} 
\begin{flalign} 
\label{eq:delay}
\text{\textbf{Bounded Delay}: $\exists K > 0$ such that for each worker $i\in[m]$,} \ \
   \tau^{(i)}_{min} \leq \tau_t^{(i)} \leq K \tau^{(i)}_{min} &&
\end{flalign}
where $\tau_{min}^{(i)}$ is the minimum delay of worker $i$. $K$ bounds the variance of the delay for each worker.
\begin{flalign} 
\label{eq:boundedIterations}
\text{\textbf{Bounded Byzantine Iterations}:~ there exists $0 \leq \lambda < \nicefrac{1}{2}$ such that $t \in [T]$:} \ \
   t_\BBB \leq \lambda t &&
\end{flalign}
where $t_\BBB$ is the total number of iterations made by Byzantine workers up to iteration $t$. \\
\textbf{Sample-Arrival Independence}: we assume that the delays in the system (i.e.~$\tau_t^{(i)}$'s) are independent of the data samples.
This is a standard assumption in asynchronous training scenarios,   see e.g.,~\cite{arjevani2020tight, aviv2021asynchronous}.

\section{Weighted Robust Aggregation Rules}
\label{sec:weightedRobustAgg}
As we have mentioned, robust aggregation rules have played a major role in designing fault-tolerant ML training methods for synchronous settings (see, e.g., \cite{allouah2023fixing, karimireddy2020byzantine, karimireddy2021learning, dahanlevy2024}). These existing aggregation rules treat inputs from all workers equally, which makes sense in synchronous cases where all workers contribute the same number of updates and data samples. Conversely, this symmetry breaks down in asynchronous settings, where faster (honest) workers contribute more updates and samples compared to slower workers.

Inspired by this asymmetry, we have identified the need to define a notion of weighted robust aggregators that generalizes the standard definition of robust aggregators. In this section, we provide such a definition, derive weighted variants of standard aggregators that satisfy our new definition, and design a generic meta-approach to derive optimal weighted aggregation rules.
Later, in Section~\ref{sec:AsynchRobust}, we demonstrate the benefits of using weighted robust aggregators as a crucial building block in designing asynchronous fault-tolerant training methods (see Alg.~\ref{alg:Asy}). 

\subsection{Robust Weighted Aggregation Framework}
Below, we generalize the definition introduced by \cite{dahanlevy2024, karimireddy2020byzantine, karimireddy2021learning} to allow and associate weights to the inputs of the robust aggregation rule, therefore allowing the aggregator to unequally treat its inputs.
\begin{definition} \textnormal{\( (c_\lambda, \lambda) \)-\textbf{weighted robust}}.
\label{def2}
Assume we have $m$ random vectors $\bx_1, \ldots, \bx_m \in \real^d$ and corresponding weights $s_1, \ldots, s_m > 0$. Also assume we have an "honest" subset \(\GGG \subseteq [m]\), implying
$\{\bx_i\}_{\in\ \GGG}$ are independent of each other.  Finally, assume that there exists $\lambda\in [0,\nicefrac{1}{2})$ such that $\sum_{i \in \GGG} s_i \geq (1-\lambda)s_{1:m}$. Moreover, assume that for any $i\in\GGG$ there exist $\rho_i\geq 0$ such that, 
\begin{gather*}
    \E \|{\bx_i - \Bar{\bx}_\GGG}\|^2\leq \rho^2_i, \quad \forall i\in\GGG~.
\end{gather*}
Then an aggregation rule $\A_\omega$ is called \( (c_\lambda, \lambda) \)-weighted robust if for any such $m$ random vectors and weights and $\lambda\geq 0$, it outputs $\hat{\bx}\gets \A_\omega(\bx_1, \ldots, \bx_m ; s_1, \ldots, s_m)$ such that,
\begin{gather*}
    \E\|{\hat{\bx} - \bar{\bx}_\GGG} \| \leq c_\lambda \rho^2
\end{gather*}
for some $c_\lambda\geq 0$. Above,  $\bar{\bx}_\GGG := \frac{1}{\sum_{i\in\GGG} s_i} \sum_{i \in \GGG} s_i{\bx}_i$, $\rho^2:=\frac{1}{\sum_{i\in\GGG} s_i} \sum_{i\in\GGG}s_i\rho_i^2$, and the expectation is w.r.t.~$\{\bx_i\}_{i=1}^m$ and (possible) randomization in the $\A_\omega$.
\end{definition}
Here, \( \lambda \) represents the fraction of the sum of the non-honest vectors' weights, unlike the unweighted definition (in synchronous cases) \citep{karimireddy2020byzantine, karimireddy2021learning, allouah2023fixing, farhadkhani2022byzantine} where it indicates the fraction of non-honest vectors. These definitions align when all weights are equal \citep{dahanlevy2024}. Similarly to the unweighted version, the optimal \( c_{\lambda} \) should be \( c_{\lambda} \leq O(\lambda) \) \citep{dahanlevy2024}.
\begin{remark}
Note that our definition is generic and may be applied in both convex and non-convex scenarios. Moreover, it is natural to consider such weighted aggregators beyond asynchronous settings.
For example, in synchronous settings where workers have varying batch sizes, weighted aggregation based on batch sizes may be more effective than uniform aggregation.
\end{remark}

Next, we present two weighted variants of standard (non-weighted) aggregators that satisfy the above definition (we defer the proof into Appendix \ref{app:robust-agg}). Table~\ref{tab:asyncFilters} summarizes their $c_\lambda$ values.

\subsection{Weighted Variant of Geometric Median and Coordinate-Wise}
\paragraph{Weighted Geometric Median (WeightedGM)}
The Weighted Geometric Median (WeightedGM) minimizes the weighted sum of Euclidean distances to a set of points. Formally, for points \(\{\bx_i\}_{i=1}^m\) and corresponding weights \(\{s_i\}_{i=1}^m\),
$\textnormal{WeightedGM} \in \arg \min_{\by \in \real^d} \sum_{i\in[m]} s_i \|\by - \bx_i\|$~.
\paragraph{Weighted Coordinate-Wise Median (WeightedCWMed)}
The Weighted Coordinate-Wise Median (WeightedCWMed) aggregates multi-dimensional data by finding the weighted median of each coordinate separately.
Thus, for given coordinate if \(\{\bx_i\}_{i=1}^m\) are sorted and weights  \(\{s_i\}_{i=1}^m\), the weighted median \(\bx_{j^*}\) is the element where:
$j^* = \arg \min_{j \in [m]}\left\{\sum_{i\in[j]} s_i > \frac{1}{2} \sum_{i\in[m]} s_i \right\}~.$
If \(\sum_{i=1}^{j} s_i = \frac{1}{2} \sum_{i=1}^{m} s_i\) for some \(j\), then: $\textnormal{WeightedMedian} = \frac{\bx_j + \bx_{j+1}}{2}$.

\begin{table}
    \centering
    \begin{tabular}{|l|c|c|c|c|}
        \hline
        Aggregation & $\omega$-GM & $\omega$-CWMed & $\omega$-GM + $\omega$-CTMA& $\omega$-CWMed + $\omega$-CTMA \\
        \hline
        \({c_\lambda}\) & \( \left (1 + \frac{\lambda}{{1-2\lambda}}\right)^2 \) & \(\left(1 + \frac{\lambda}{{1-2\lambda}}\right)^2 \) & \(  \lambda\left(1 + \frac{\lambda}{{1-2\lambda}}\right)^2 \) & \(\lambda \left(1 + \frac{\lambda}{{1-2\lambda}}\right)^2 \)  \\
        \hline
    \end{tabular}
    \caption{Summary of weighted aggregation rules and their respective \(c_\lambda\) values.}
    \label{tab:asyncFilters}
\end{table}

\subsection{Weighted Centered Trimmed Meta Aggregator ($\omega$-CTMA)}
Table \ref{tab:asyncFilters} illustrates that  \( \omega \)-GM and \( \omega \)-CWMed fail to achieve the desired optimal $c_\lambda = O(\lambda)$; typically for \( \lambda \leq \nicefrac{1}{3} \), their \( c_{\lambda} \) remains \( \leq O(1) \). To address this suboptimality, we propose \(\omega\)-CTMA, a weighted extension of the Centered Trimmed Meta Aggregator (CTMA) \citep{dahanlevy2024}. This extension enables us to achieve the optimal bound \(c_\lambda \leq O(\lambda)\) for \(\lambda \leq \nicefrac{1}{3}\) (see Table \ref{tab:asyncFilters}).

The \(\omega\)-CTMA algorithm (Algorithm \ref{alg:CTMA}) operates on a set of vectors along with their associated weights, a threshold \(\lambda\in[0,\nicefrac{1}{2})\), and a \((c_\lambda, \lambda)\)-weighted robust aggregator. It sorts the distances between each vector and the weighted robust aggregator, trims the set based on the threshold to satisfy \(\sum_{i \in S} s_i = (1 - \lambda) s_{1:m}\), and calculates a weighted average of the vectors, excluding outliers based on their proximity to an anchor point—the weighted robust aggregator.

\begin{algorithm}[t]
\caption{Weighted Centered Trimmed Meta Aggregator (\(\omega\)-CTMA)}\label{alg:CTMA}
\begin{algorithmic}[1] 
    \State \textbf{Input:} Set of vectors \(\{\bx_i\}_{i=1}^m\), weights \(\{s_i\}_{i=1}^m\), threshold parameter \(\lambda \in[0,\nicefrac{1}{2})\),
    \State \phantom{\textbf{Input:}} \((c_\lambda,\lambda)\)-weighted robust aggregated vector \(\bx_0 \leftarrow \A_\omega(\{\bx_i\}_{i=1}^m;\{s_i\}_{i=1}^m)\). 
    \State Sort the sequence \(\{ \|\bx_i - \bx_0\| \}_{i=1}^m\) in non-decreasing order, and then reindex \(\{\bx_i\}_{i\in[m]}\) and their corresponding weights \(\{s_i\}_{i\in[m]}\) according to this new order.
    \State Define \( S \gets \) set of indices corresponding to the first \( j^* \) elements in the sorted sequence, where \( j^* \) is the smallest \( j \in [m] \) for which \( \sum_{i\in[j]}s_i \geq (1-\lambda) \sum_{i\in[m]} s_i \).
    \State Set \( s_{m+1} \leftarrow (1-\lambda) \sum_{i\in[m]} s_i - \sum_{i\in[j^*-1]}s_i \), \quad \( \bx_{m+1} \leftarrow \bx_{j^*} \), \quad \( S \leftarrow (S\backslash \{j^*\}) \cup \{m+1\}  \).
    \State Compute the weighted sum:~~~ 
    \(\hat{\bx} \gets (\sfrac{1}{\sum_{i\in S} s_i})\sum_{i\in S} s_i \bx_i \).
    \State \textbf{Output:} \(\hat{\bx}\)
\end{algorithmic}
\end{algorithm}

\begin{lemma}
\label{lem:CTMA}
Under the assumptions outlined in Definition \ref{def2}, if $\omega$-CTMA receives a $(c_\lambda, \lambda)$- weighted robust aggregator, $\A_\omega$; then the output of $\omega$-CTMA, $\hat{\bx}$, is  $(60\lambda(1+c_\lambda), \lambda)$-robust.
\end{lemma}
For the complete analysis, please refer to Appendix \ref{app:ctma}. Like CTMA \citep{dahanlevy2024}, $\omega$-CTMA is highly efficient, with a computational complexity of \(O(dm + m\log{m})\), similar to \(\omega\)-GM, \(\omega\)-CWMed, and weighted averaging, differing by at most an additional logarithmic factor.

\section{Asynchronous Robust Training}
\label{sec:AsynchRobust}
We leverage the \(\mu^2\)-SGD algorithm \citep{levy2023mu}, a double momentum mechanism that enhances variance reduction. By seamlessly incorporating our weighted robust framework as a black box into the \(\mu^2\)-SGD, we derive an optimal asynchronous Byzantine convergence rate.

\paragraph{$\mu^2$-SGD:}
The \(\mu^2\)-SGD is a variant of standard SGD, incorporating several key modifications in its update rule:
\begin{gather*}
    \label{eq:mu2sgd}
    \bw_{t+1}=\Pi_{\K}\left({\bw_{t}- \eta\alpha_{t} \bd_{t}}\right), \quad
    \bx_{t+1}=\frac{1}{\alpha_{1:t+1}} \sum_{k\in[t+1]} \alpha_k \bw_k; \quad \bw_1=\bx_1\in\K, \ \forall t>1.
\end{gather*}
Here, \(\{\alpha_t>0\}_t\) are importance weights that emphasize different update steps, with \(\alpha_t\propto t\) to place more weight on recent updates. The sequence \(\{\bx_t\}_t\) represents weighted averages of the iterates \(\{\bw_t\}_t\), and \(\bd_t\) is an estimate of the gradient at the average point, \(\nabla f(\bx_t)\), differing from standard SGD, which estimates gradients at the iterates, \(\nabla f(\bw_t)\).

This approach relates to Anytime-GD \citep{cutkosky2019anytime}, which is strongly connected to momentum and acceleration concepts \citep{cutkosky2019anytime, kavis2019unixgrad}. While the stochastic version of Anytime-GD typically uses the estimate \(\nabla f(\bx_t;\bz_t)\), \(\mu^2\)-SGD employs a variance reduction mechanism to produce a \emph{corrected momentum} estimate \(\bd_t\) \citep{cutkosky2019momentum}. Specifically, \(\bd_1 = \nabla f(\bx_1;\bz_1)\), and for \(t > 2\):
\begin{equation*}
\bd_{t}=\nabla f(\bx_{t};\bz_{t}) + (1-\beta_{t})(\bd_{t-1} - \nabla f(\bx_{t-1};\bz_{t})).
\end{equation*}
Here, \(\beta_t \in [0,1]\) are \emph{corrected momentum} weights. It can be shown by induction that \(\E[\bd_{t}] = \E[\nabla f(x_t)]\); however, in general, \(\E[\bd_{t} \vert x_t] \neq \nabla f(x_t)\), unlike standard SGD estimators. Nevertheless, \citep{levy2023mu} demonstrates that choosing \emph{corrected momentum} weights \(\beta_t := 1/t\) results in significant error reduction, with \(\E\|\varepsilon_t\|^2 := \E\|\bd_t - \nabla f(\bx_t)\|^2 \leq O(\tsigma^2/t)\) at step \(t\), where \(\tsigma^2 \leq O(\sigma^2 + D^2K^2 \sigma_L^2)\). This indicates that variance decreases with \(t\), contrasting with standard SGD where the variance \(\E\|\varepsilon^{\textnormal{SGD}}_t\|^2 := \E\|\bg_t - \nabla f(\bx_t)\|^2\) remains uniformly bounded.

\subsection{Asynchronous Robust $\mu^2$-SGD}
\begin{algorithm}[t]
\caption{Asynchronous Robust \(\mu^2\)-SGD}\label{alg:Asy}
\begin{algorithmic}[1] 
    \State \textbf{Input:} learning rate \(\eta_t > 0\), starting point \(\bx_{1} \in \K\), number of steps \(T\), importance weights \(\{\alpha_t\}_t\), momentum correction weights \(\{\beta_t\}_t\),  \((c_\lambda, \lambda)\)-robust weighted aggregation function \(\mathcal{A}_\omega\).
    \State \textbf{Initialize:} \(\forall i \in [m]\), set \(s_0^{(i)} = 0\). Set \(\bw_1 = \bx_1\).  Each honest worker \(i \in \GGG\) draws \(\bz^{(i)} \sim \mathcal{D}\) and set \(\bd^{(i)}_1 = \nabla f(\bx_1; \bz^{(i)})\).
    \For{\(t = 1\) \textbf{to} \(T\)} \Comment{\textit{Server update}}
        \State Receive \(\bd_{t-\tau_t}\) from worker \(i \in [m]\) and update:
        \State \hspace{\algorithmicindent} \(\bd_t^{(i)} = \bd_{t-\tau_t}\), \ \(s_t^{(i)} = s_{t-1}^{(i)} + 1\); \quad \(\forall \underline{j \neq i}\): set  \(s_t^{(j)} = s_{t-1}^{(j)}\), and for \(\underline{t>1}\): \(\bd_t^{(j)} = \bd_{t-1}^{(j)}\);
        \State Update server model:
        \State \hspace{\algorithmicindent} 
        \(\bw_{t+1} = \Pi_{\K}\left(\bw_{t} - \eta_t \alpha_t \mathcal{A}_\omega(\{\bd_t^{(j)}, s_t^{(j)}\}_{j=1}^m)\right)\)~,~~\&~~ \(\bx_{t+1} = \frac{1}{\sum_{k=1}^{t+1} \alpha_k} \sum_{k=1}^{t+1} \alpha_k \bw_k\)
        \State Send \(\bx_t\)  to worker \(i\). If \(i\) is an honest worker, it performs the following update:
        \State \hspace{\algorithmicindent} Worker \(i\) draws \(\bz_t \sim \D\), computes \(\bg_t = \nabla f(\bx_t; \bz_t)\)~,~~\&~~\(\Tilde{\bg}_{t-\tau_t} = \nabla f(\bx_{t-\tau_t}; \bz_t)\)~,
        \State \hspace{\algorithmicindent} and updates:
        \(\bd_t = \bg_t + (1 - \beta_t)(\bd_{t-\tau_t} - \Tilde{\bg}_{t-\tau_t})\) \Comment{\textit{Worker update}}
    \EndFor
    \State \textbf{Output:} \(\bx_T\)
\end{algorithmic}
\end{algorithm}
Building upon these, we integrate the \(\mu^2\)-SGD with a $(c_\lambda, \lambda)$-weighted robust aggregator $\A_\omega$, as described in Alg. \ref{alg:Asy}. At each iteration $t\in[T]$, the global $\mathcal{PS}$ receives an output from a certain worker and aggregates all workers' recent updates \( \left\{\bd^{(i)}_{{t}}\right\}_{i=1}^m\) by employing weights accordingly to the number of updates of each worker \( \left\{s^{(i)}_{{t}}\right\}_{i=1}^m\). An honest worker $i$ arriving at iteration $t$ returns its corrected momentum $\bd_t^{(i)}$ to the $\mathcal{PS}$, computed as:
\begin{align*}
     \bd_t^{(i)}=\bd_{t-\tau_t}=\bg_{t-\tau_t}+(1-\beta_{t-\tau_t})(\bd_{t-\tau_t-\tau_{t-\tau_t}}-\Tilde{\bg}_{t-\tau_t-\tau_{t-\tau_t}})~,
\end{align*}
where $\bg_{t}:=\nabla f(\bx_{t};\bz_{t})$, and $\Tilde{\bg}_{t-\tau_t}:=\nabla f(\bx_{t-\tau_t};\bz_{t})$. 
Afterwards, the $\mathcal{PS}$ performs the AnyTime update step as follows:
\begin{align*}
    \bw_{t+1}=\Pi_{\K}\left({\bw_{t}- \eta\alpha_{t} \mathcal{A}_\omega(\{\bd_t^{(i)}, s_t^{(i)}\}_{i=1}^m)}\right), \
\bx_{t+1}=\frac{1}{\alpha_{1:t+1}} \sum_{k\in[t+1]} \alpha_k \bw_k~.
\end{align*}
In the spirit of \cite{levy2023mu,dahanlevy2024}, we suggest employing \(\beta_t := 1/s_t\), which effectively considers the entire individual gradient's history of each worker; this translates to a stochastic error bound of \(\E\|\varepsilon^{(i)}_t\|^2 \leq O(\tilde{\sigma}^2/s_t)\) for an honest worker $i$ arriving at iteration $t$. To achieve an error corresponding to the total number of honest iterations $t_\GGG$, specifically \(\E\|\varepsilon_t\|^2 \leq O(\tilde{\sigma}^2/t_\GGG)\), as in the non-distributed setting \citep{levy2023mu},  a weighted collective error across all honest workers should be considered with weights determined by the number of honest worker arrivals, as detailed in Theorem \ref{thm:MainAsy}. The unique characteristics of \(\mu^2\)-SGD make it well-suited for the asynchronous Byzantine setting, where \(\lambda < \nicefrac{1}{2}\) relates to the fraction of Byzantine iterations. The total iteration number \(t\) matches the sum of the workers' frequencies (\(\sum_{i \in [\GGG]} s_t^{(i)} = t_\GGG\)), aligning with the weighted robust definition in Definition \ref{def2}. Using other approaches like momentum \citep{karimireddy2020byzantine, karimireddy2021learning, allouah2023fixing} is less straightforward in the asynchronous Byzantine setting, where their stochastic error does not align with the entire iterations, posing an additional challenge.

\begin{remark}[Memory and Computational Overhead of Algorithm \ref{alg:Asy}]
    Algorithm \ref{alg:Asy} incurs additional memory and computational costs compared to the asynchronous Byzantine-free setting \citep{arjevani2020tight}, where the server stores only one worker's output and the global model. Algorithm \ref{alg:Asy} stores the latest outputs from all workers, increasing memory usage to \( O(dm) \). Robust aggregation methods like $\omega$-CWMed \citep{yin2018byzantine} and ($\epsilon$-approximate) $\omega$-GM \citep{chen2017distributed, acharya2022robust} add a computational cost of \( O(dm\log m) \) and \( O(dm+d\epsilon^{-2}) \), respectively. This is in contrast to Byzantine-free settings where worker outputs are used directly without aggregation. Comparable overheads are observed in synchronous Byzantine-resilient methods, which similarly aggregate outputs from all workers. This reflects a necessary trade-off: achieving robustness inherently requires leveraging information from all workers to counteract the influence of potentially faulty ones.
\end{remark}

\begin{theorem}
\label{thm:MainAsy}
For a convex set \(\K\) with bounded diameter \(D\) and a function \(f:\K\mapsto\real\), and assume the assumptions in Equations~\eqref{eq:bounded-variance},\eqref{eq:Main},\eqref{eq:sigmal}. Then Alg.~\ref{alg:Asy} with parameters \(\{\alpha_t = t\}_t\) and \(\{\beta_t=1/s_{t}\}_t\) ensures the following for every \(t\in [T]\) and each honest worker \(i\in\GGG\):
\begin{gather*}
       \E\left\|{\varepsilon^{(i)}_{t}}\right\|^2 \leq  {\frac{\Tilde{\sigma}^2}{s^{(i)}_{t}}}, \quad
       \E\left\|\frac{1}{\sum_{i\in\GGG} s_t^{(i)}}{\sum_{i\in \GGG} s_{t}^{(i)} \varepsilon^{(i)}_{t}}\right\|^2 \leq \frac{\Tilde{\sigma}^2}{t_\GGG}~,
    \end{gather*}
where \(\varepsilon^{(i)}_{t} = \bd^{(i)}_{t}-\nabla f(\bx^{(i)}_{t})\), \(\tsigma^2=2\sigma^2 + 32D^2K^2 \sigma_L^2\), and $t_\GGG$ is the total number of honest iterations up to the $t^{\textnormal{th}}$ iteration.
\end{theorem}

\begin{proof}[Proof Sketch of Thm.~\ref{thm:MainAsy}]

The complete analysis is provided in App. \ref{app:main}. It involves several key steps for an honest $i$ worker who arrives at iteration $t$:
\begin{enumerate}[leftmargin=1em]
    \item
   Following Lemma \ref{lem:PointDistworker}, the distance between successive query points:$\|{\bx}^{(i)}_{t} - {\bx}^{(i)}_{t-\tau_t}\| \leq \frac{4K}{s^{(i)}_t-1} D$
   \item We analyze the recursive dynamics of the error term ${\varepsilon_{t}^{(i)}}$ by setting $\beta_t = \frac{1}{s^{(i)}_t}$ and obtain:
   $$s^{(i)}_t{\varepsilon}^{(i)}_{t} = (\bg^{(i)}_{t} - \nabla f({\bx}^{(i)}_{t})) + (s^{(i)}_t-1)Z^{(i)}_{t} + (s^{(i)}_t-1){\varepsilon}^{(i)}_{t-\tau_t}~,
   $$
   where $Z^{(i)}_{t}:=\bg_{t}^{(i)}-\nabla f({\bx}_{t}^{(i)})-(\Tilde{\bg}_{t-\tau_t}^{(i)}-\nabla f({\bx}_{t-\tau_t}^{(i)}))$. Unrolling this recursion provides an explicit expression: 
   $s_t^{(i)}{\varepsilon}^{(i)}_{t} = \sum_{k\in[s^{(i)}_t]}\MMM^{(i)}_k~,$
   where $\MMM^{(i)}_{s^{(i)}_t}:=\bg^{(i)}_{t}-\nabla f({\bx}^{(i)}_{t})+(s_t-1)Z^{(i)}_{t}$; thus, $\{ \MMM_k^{(i)}\}_{k\in[s_t^{(i)}]}$  is a martingale difference sequence.
   \item  Employing the above with Eq. \eqref{eq:bounded-variance} and \eqref{eq:sigmal}, we have: $\E\|\MMM^{(i)}_k\|^2 \leq 2\sigma^2 + 32D^2K^2\sigma_L^2 = \tsigma^2$.
   \item Leveraging the properties of a martingale difference sequence, we have:
   \begin{gather*}
    \E\left\|{s^{(i)}_t\varepsilon_{t}^{(i)}} \right\|^2 =\E\left\|\sum_{k\in\left[s^{(i)}_t\right]}\MMM_k^{(i)}\right\|^2  =\sum_{k\in\left[s^{(i)}_t\right]}\E\left\|\MMM_k^{(i)}\right\|^2  \leq \tsigma^2 s_t^{(i)} ~,\\
    \E\left\|{\sum_{i\in {\GGG}}s^{(i)}_t\varepsilon_t^{(i)}} \right\|^2 = \E
        \left\|\sum_{i\in {\GGG} }\sum_{k\in\left[s^{(i)}_t\right]} \MMM_k^{(i)}\right\|^2 = \sum_{i\in {\GGG} }\sum_{k\in\left[s^{(i)}_t\right]}\E
        \left\|\MMM_k^{(i)}\right\|^2
         \leq \tsigma^2 \sum_{i\in {\GGG}} s^{(i)}_t 
        = \tsigma^2t_{\GGG}~.
\end{gather*}
\end{enumerate}
\end{proof}
\begin{remark}
    Compared to synchronous scenarios \citep{levy2023mu, dahanlevy2024}, the variance \(\tsigma\) in Thm. \ref{alg:Asy} additionally includes the variance in the delay, denoted as \(K\) (Eq. \eqref{eq:delay}). In balanced scheduling methods, like Round Robin \citep{langford2009slow}, the impact of \(K\) on the error becomes minor, as the delay \(\tau^{(i)}_t=m\) is constant. In the case of constant delays, the factor \(K\) equals $1$.  
\end{remark}

\begin{lemma}
\label{lem:asyncFilter}
Let $\A_\omega$ be $(c_\lambda,\lambda)$-weighted robust aggregation rule and let $f:\K\mapsto\real$, where $\K$ is a convex set with bounded diameter $D$, and presume that the assumption in Equations~\eqref{eq:bounded-variance},\eqref{eq:Main},\eqref{eq:sigmal} hold.  Then invoking Alg.~\ref{alg:Asy} with $\{\alpha_t = t\}_t$ and $\{\beta_t=1/s_t\}_t$, ensures the following for any $t\in[T]$,
\begin{align*}
    \E\left\|\hat{\bd}_t-\nabla f(\bx_t) \right\|^2 \leq O\left(\underbrace{\frac{\tsigma^2}{t} + \frac{c_\lambda m\tsigma^2}{t}}_{\mathrm{Variance}} + \underbrace{\frac{(\tau^{max}_{t}DL)^2}{t^2} + \frac{c_\lambda(\tau^{max}_{t}DL)^2}{t^2}}_{\mathrm{Bias}}\right)
\end{align*}
where $\hat{\bd}_t=\mathcal{A}_\omega(\{\bd_t^{(i)}, s_t^{(i)}\}_{i=1}^m)$, $\tau_t^{max}=\max_{i\in[m]} \{\tau_t^{(i)}\}$, and \(\tsigma^2=2\sigma^2 + 32D^2K^2 \sigma_L^2\).
\end{lemma}
Lemma \ref{lem:asyncFilter} shows that the error between our gradient estimator \(\hat{\bd}_t\) and the true gradient includes a bias term arising from the aggregation of delayed momentums. This is in contrast to the synchronous scenario \citep{dahanlevy2024} where the error is solely variance-dependent without any bias component. However, this bias does not affect the overall excess loss (Theorem \ref{thm:AsymuSGD}), which remains comparable to the optimal rate achieved in synchronous Byzantine settings (see Remark \ref{remark:sync}).

By integrating the weighted robust aggregator with the double momentum mechanism, we achieve the optimal convergence rate for the first time in an asynchronous Byzantine setting—a significant advancement over previous efforts.
\begin{theorem}[Asynchronous Byzantine $\mu^2$-SGD Guarantees]\label{thm:AsymuSGD}
Let $\A_\omega$ be $(c_\lambda,\lambda)$-weighted robust aggregation rule and let $f$ be a convex function. 
Also, let us make the same assumptions as in Thm.~\ref{thm:MainAsy}, and let us denote $G^*:=\| \nabla f(\bx^*)\|$, where $ \bx^* \in \arg\min_{\bx\in\K} f(\bx)$.
Then invoking Alg.~\ref{alg:Asy} with $\{\alpha_t = t\}_t$ and $\{\beta_t=1/s_{t}\}_t$, and using a learning rate $\eta\leq 1/4LT$ guarantees,
\begin{align*}
    \E\left[f(\bx_T) - f(\bw^*)\right] 
    & \leq O\left( \frac{G^*D+LD^2\mu_{max}\sqrt{1+c_\lambda}}{T} + \frac{D{\tsigma}\sqrt{1+c_\lambda m}}{\sqrt{T}}\right)
\end{align*}
where $\tsigma^2 = 2\sigma^2 + 32D^2K^2 \sigma_L^2$, \ $\mu_{max}=\frac{1}{T}\sum_{t\in[T]} \tau_t^{max}$, \ and $\tau_t^{max}=\max_{i\in[m]} \{\tau_t^{(i)}\}$.
\end{theorem}
\begin{remark}
    In the absence of Byzantine iterations (\(\lambda=0\)), the parameter \(c_\lambda\) of a \((c_\lambda, \lambda)\)-weighted robust aggregator can diminish to 0 when we use $\omega$-CTMA (see Table \ref{tab:asyncFilters}). This aligns with the asynchronous SGD analysis \citep{arjevani2020tight} and represents the first work to achieve optimal convergence without Byzantine workers compared to previous efforts \citep{yang2021basgd, yang2023buffered, fang2022aflguard, zhu2023asynchronous, damaskinos2018asynchronous, xie2020zeno++}.
\end{remark}
\begin{remark}
    Unlike previous works \citep{yang2021basgd, yang2023buffered, fang2022aflguard, zhu2023asynchronous, damaskinos2018asynchronous, xie2020zeno++}, our convergence rate is independent of data dimensionality $d$ and is sublinear at $T$, even in the presence of Byzantine workers.
\end{remark}
\begin{remark}
\label{remark:sync}
This result is consistent with the synchronous scenario \citep{dahanlevy2024}, where the delay is constant \(\tau_t = m\) as in Round Robin \citep{langford2009slow}. In this case, the proportion of Byzantine workers is \(\lambda\), and the asynchronous excess loss is \(\leq O\left(\frac{LD^2m}{T}+\frac{D\tilde{\sigma}\sqrt{1+c_\lambda m}}{\sqrt{T}}\right)\). In comparison to the synchronous case, where \(m\) workers perform \(R\) rounds, here we make \(R\) query point updates and \(T = Rm\) data-samples, resulting in synchronous excess loss \(\leq O\left(\frac{LD^2}{R}+\frac{D\tilde{\sigma}\sqrt{1/m+c_\lambda}}{\sqrt{R}}\right) = O\left(\frac{LD^2m}{T}+\frac{D\tilde{\sigma}\sqrt{1+mc_\lambda}}{\sqrt{T}}\right)\) \citep{dahanlevy2024}.
\end{remark}

\section{Experiments}
To evaluate the effectiveness of our proposed approach, we conducted experiments on MNIST \citep{lecun2010mnist} and CIFAR-10 \citep{krizhevsky2014cifar} datasets—two recognized benchmarks in image classification tasks. We employed a two-layer convolutional neural network architecture for both datasets, implemented using the PyTorch framework. The training was performed using the cross-entropy loss function, and all computations were executed on an NVIDIA L40S GPU. To ensure the robustness of our findings, each experiment was repeated with three different random seeds, and the results were averaged accordingly. Our experimental results demonstrate consistent performance across both datasets. Further details about the experimental setup and the complete results are provided in Appendix  \ref{app:exp}.

{\textbf{Weighted vs. Non-Weighted Robust Aggregators}.} We evaluated the test accuracy of weighted and non-weighted robust aggregators in imbalanced asynchronous Byzantine environments. Our experiments show that weighted robust aggregators consistently achieved higher test accuracy than the non-weighted ones (see Figure \ref{fig:weighted_vs_nonweighted} and Figure \ref{fig:weighted}). This highlights the benefit of prioritizing workers who contribute more updates in asynchronous setups.
\begin{figure}[H]
\centering
\includegraphics[width=0.7\linewidth]{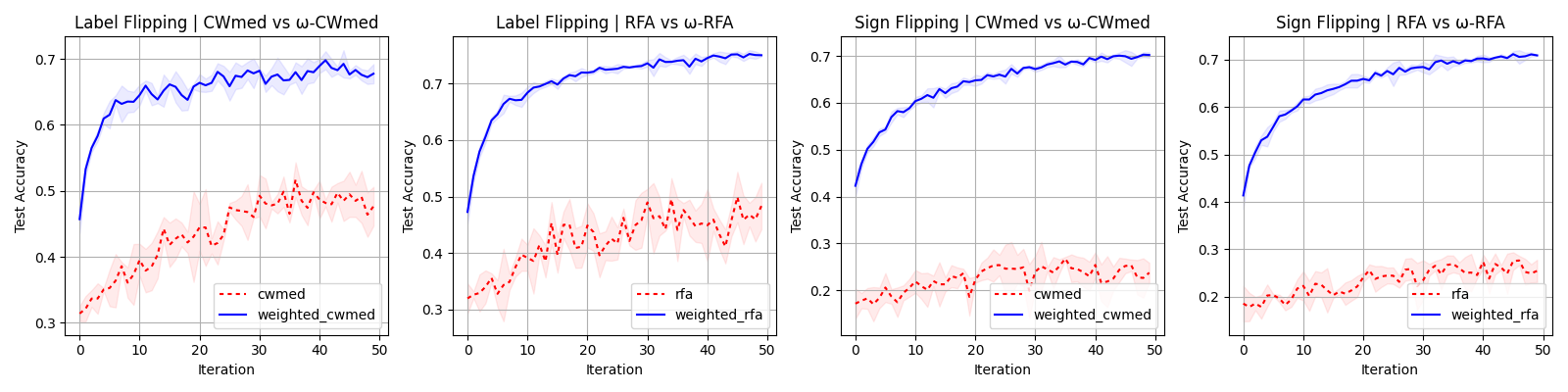}
\caption{\smaller \textbf{CIFAR-10}. \textbf{Test Accuracy of Weighted vs. Non-Weighted Robust Aggregators}. This scenario involves 17 workers, including 8 Byzantine workers, with workers' arrival probabilities proportional to the square of their IDs. We used the $\mu^2$-SGD in this scenario. {Left}: \textit{label flipping}, $\lambda = 0.3$. {Right}: \textit{sign flipping}, $\lambda = 0.4$.}
\label{fig:weighted_vs_nonweighted}
\end{figure}

{\textbf{Effectiveness of $\omega$-CTMA}.} We evaluated the test accuracy of weighted robust aggregators with and without the integration of $\omega$-CTMA, as shown in Figure \ref{fig:ctma_comparison} and Figure \ref{fig:ctma}. The results demonstrate that $\omega$-CTMA can enhance the performance of weighted robust aggregators in various Byzantine scenarios for both datasets.
\begin{figure}[H]
\centering
\includegraphics[width=0.7\linewidth]{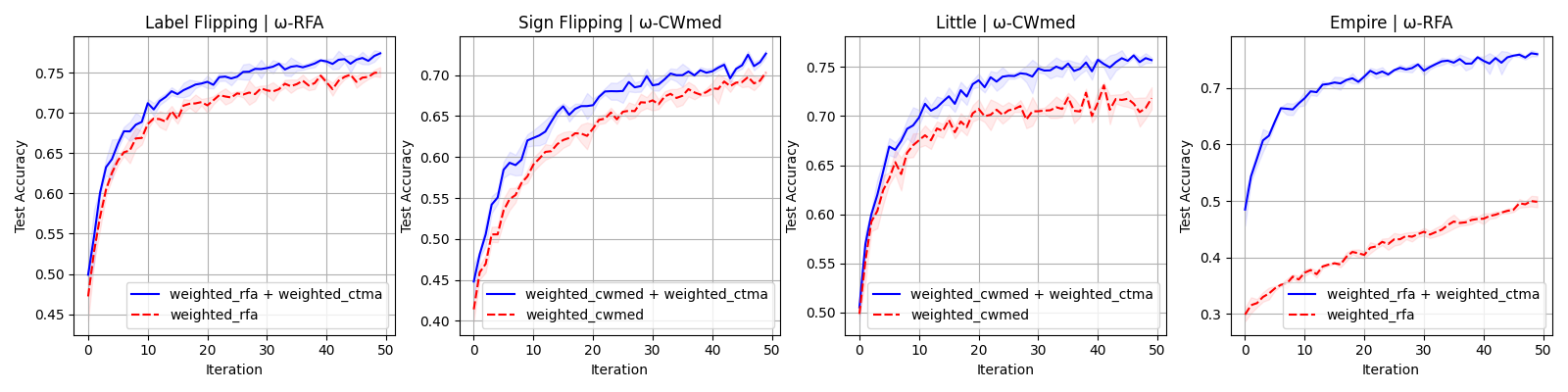}
\caption{\smaller \textbf{CIFAR-10}. \textbf{Test Accuracy Comparison of Weighted Robust Aggregators With and Without $\omega$-CTMA}. This scenario involves 9 workers, including either 1 or 3 Byzantine workers. The arrival probabilities of workers are proportional to their IDs, and we employed $\mu^2$-SGD. On the \textit{left}, the \textit{label flipping} and \textit{sign flipping} attacks are depicted with $\lambda = 0.3$ and $\lambda = 0.4$, respectively, using 3 Byzantine workers. On the \textit{right}, the \textit{little} attack is shown with $\lambda = 0.1$ and 1 Byzantine worker, and the \textit{empire} attack is shown with $\lambda = 0.4$ and 3 Byzantine workers.}
\label{fig:ctma_comparison}
\end{figure}

{\textbf{Performance of $\mu^2$-SGD vs. Standard Momentum and SGD}.} We evaluated the test accuracy of $\mu^2$-SGD in comparison to standard momentum \citep{polyak1964some} and SGD \citep{eon1998online} within an asynchronous Byzantine setup. Figure \ref{fig:optimizers} and Figure \ref{fig:optimizer} show that $\mu^2$-SGD performs on par with standard momentum, while SGD generally exhibits poorer performance relative to both. These results underscore the importance of utilizing historical information when addressing Byzantine scenarios.
\begin{figure}[H]
\centering
\includegraphics[width=0.7\linewidth]{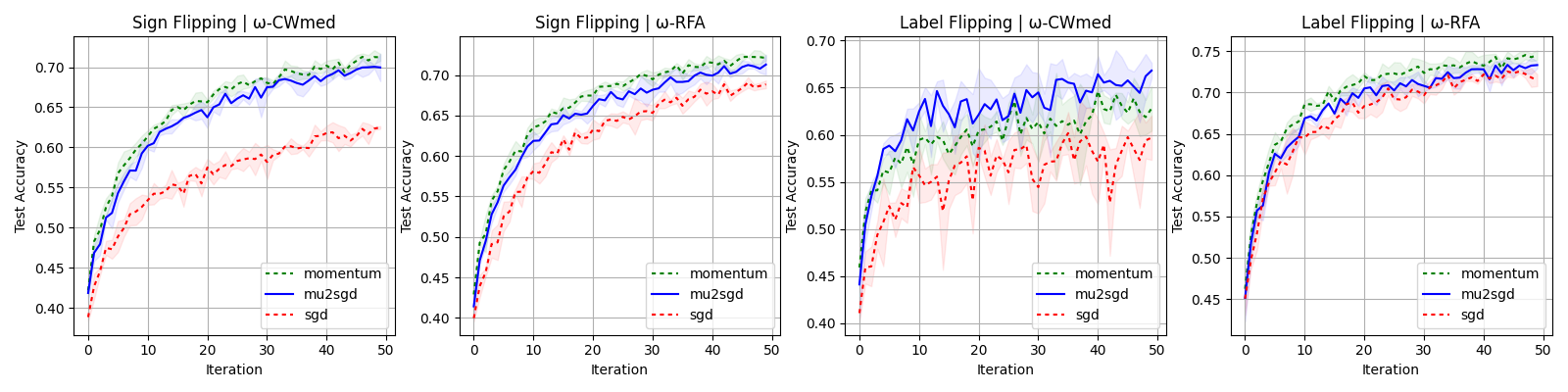}
\caption{\smaller \textbf{CIFAR-10}. \textbf{Test Accuracy Comparison Among Different Optimizers}. This scenario involves 9 workers (4 Byzantine) with $\lambda=0.4$, and workers' arrival probabilities are proportional to their IDs. {Left}: \textit{sign flipping}. {Right}: 
 \textit{label flipping}. }
\label{fig:optimizers}
\end{figure}

\section*{Conclusions and Future Work}
This paper shows that using a double momentum approach, which incorporates the entire history of each honest worker, improves the stochastic error bound to be proportional to the total number of updates when considering their weighted average in asynchronous settings. By integrating this method with a weighted robust framework, \(\mu^2\)-SGD achieves an optimal convergence rate, making it particularly effective for asynchronous Byzantine environments. However, integrating other optimization algorithms, like momentum, into this weighted robust framework can be challenging, as they do not achieve an error bound proportional to the total number of updates and may complicate the adjustment of weights based on the update count. This highlights the need for further research to adapt different methods to the spirit of this framework in non-convex and convex settings. 

\section*{Acknowledgement}
This research was partially supported by Israel PBC-VATAT, the Technion Artificial Intelligent Hub (Tech.AI), and the Israel Science Foundation (grant No. 3109/24).

\bibliographystyle{plainnat} 
\bibliography{bib}

\begin{thebibliography}{40}
\providecommand{\natexlab}[1]{#1}
\providecommand{\url}[1]{\texttt{#1}}
\expandafter\ifx\csname urlstyle\endcsname\relax
  \providecommand{\doi}[1]{doi: #1}\else
  \providecommand{\doi}{doi: \begingroup \urlstyle{rm}\Url}\fi

\bibitem[Acharya et~al.(2022)Acharya, Hashemi, Jain, Sanghavi, Dhillon, and Topcu]{acharya2022robust}
Anish Acharya, Abolfazl Hashemi, Prateek Jain, Sujay Sanghavi, Inderjit~S Dhillon, and Ufuk Topcu.
\newblock Robust training in high dimensions via block coordinate geometric median descent.
\newblock In \emph{International Conference on Artificial Intelligence and Statistics}, pages 11145--11168. PMLR, 2022.

\bibitem[Alistarh et~al.(2018)Alistarh, Allen-Zhu, and Li]{alistarh2018byzantine}
Dan Alistarh, Zeyuan Allen-Zhu, and Jerry Li.
\newblock Byzantine stochastic gradient descent.
\newblock \emph{Advances in Neural Information Processing Systems}, 31, 2018.

\bibitem[Allen-Zhu et~al.(2020)Allen-Zhu, Ebrahimian, Li, and Alistarh]{allen2020byzantine}
Zeyuan Allen-Zhu, Faeze Ebrahimian, Jerry Li, and Dan Alistarh.
\newblock Byzantine-resilient non-convex stochastic gradient descent.
\newblock \emph{arXiv preprint arXiv:2012.14368}, 2020.

\bibitem[Allouah et~al.(2023)Allouah, Farhadkhani, Guerraoui, Gupta, Pinot, and Stephan]{allouah2023fixing}
Youssef Allouah, Sadegh Farhadkhani, Rachid Guerraoui, Nirupam Gupta, Rafa{\"e}l Pinot, and John Stephan.
\newblock Fixing by mixing: A recipe for optimal byzantine ml under heterogeneity.
\newblock In \emph{International Conference on Artificial Intelligence and Statistics}, pages 1232--1300. PMLR, 2023.

\bibitem[Arjevani et~al.(2020)Arjevani, Shamir, and Srebro]{arjevani2020tight}
Yossi Arjevani, Ohad Shamir, and Nathan Srebro.
\newblock A tight convergence analysis for stochastic gradient descent with delayed updates.
\newblock In \emph{Algorithmic Learning Theory}, pages 111--132. PMLR, 2020.

\bibitem[Aviv et~al.(2021)Aviv, Hakimi, Schuster, and Levy]{aviv2021asynchronous}
Rotem~Zamir Aviv, Ido Hakimi, Assaf Schuster, and Kfir~Yehuda Levy.
\newblock Asynchronous distributed learning: Adapting to gradient delays without prior knowledge.
\newblock In \emph{International Conference on Machine Learning}, pages 436--445. PMLR, 2021.

\bibitem[Baruch et~al.(2019)Baruch, Baruch, and Goldberg]{baruch2019little}
Gilad Baruch, Moran Baruch, and Yoav Goldberg.
\newblock A little is enough: Circumventing defenses for distributed learning.
\newblock \emph{Advances in Neural Information Processing Systems}, 32, 2019.

\bibitem[Blanchard et~al.(2017)Blanchard, El~Mhamdi, Guerraoui, and Stainer]{blanchard2017machine}
Peva Blanchard, El~Mahdi El~Mhamdi, Rachid Guerraoui, and Julien Stainer.
\newblock Machine learning with adversaries: Byzantine tolerant gradient descent.
\newblock \emph{Advances in neural information processing systems}, 30, 2017.

\bibitem[Chen et~al.(2017)Chen, Su, and Xu]{chen2017distributed}
Yudong Chen, Lili Su, and Jiaming Xu.
\newblock Distributed statistical machine learning in adversarial settings: Byzantine gradient descent.
\newblock \emph{Proceedings of the ACM on Measurement and Analysis of Computing Systems}, 1\penalty0 (2):\penalty0 1--25, 2017.

\bibitem[Cohen et~al.(2021)Cohen, Daniely, Drori, Koren, and Schain]{cohen2021asynchronous}
Alon Cohen, Amit Daniely, Yoel Drori, Tomer Koren, and Mariano Schain.
\newblock Asynchronous stochastic optimization robust to arbitrary delays.
\newblock \emph{Advances in Neural Information Processing Systems}, 34:\penalty0 9024--9035, 2021.

\bibitem[Cutkosky(2019)]{cutkosky2019anytime}
Ashok Cutkosky.
\newblock Anytime online-to-batch, optimism and acceleration.
\newblock In \emph{International conference on machine learning}, pages 1446--1454. PMLR, 2019.

\bibitem[Cutkosky and Orabona(2019)]{cutkosky2019momentum}
Ashok Cutkosky and Francesco Orabona.
\newblock Momentum-based variance reduction in non-convex sgd.
\newblock \emph{Advances in neural information processing systems}, 32, 2019.

\bibitem[Dahan and Levy(2024)]{dahanlevy2024}
Tehila Dahan and Kfir~Yehuda Levy.
\newblock Fault tolerant ml: Efficient meta-aggregation and synchronous training.
\newblock In \emph{Forty-first International Conference on Machine Learning}, 2024.

\bibitem[Damaskinos et~al.(2018)Damaskinos, Guerraoui, Patra, Taziki, et~al.]{damaskinos2018asynchronous}
Georgios Damaskinos, Rachid Guerraoui, Rhicheek Patra, Mahsa Taziki, et~al.
\newblock Asynchronous byzantine machine learning (the case of sgd).
\newblock In \emph{International Conference on Machine Learning}, pages 1145--1154. PMLR, 2018.

\bibitem[Dekel et~al.(2012)Dekel, Gilad-Bachrach, Shamir, and Xiao]{dekel2012optimal}
Ofer Dekel, Ran Gilad-Bachrach, Ohad Shamir, and Lin Xiao.
\newblock Optimal distributed online prediction using mini-batches.
\newblock \emph{Journal of Machine Learning Research}, 13\penalty0 (1), 2012.

\bibitem[El~Mhamdi et~al.(2021)El~Mhamdi, Guerraoui, and Rouault]{el2021distributed}
El~Mahdi El~Mhamdi, Rachid Guerraoui, and S{\'e}bastien Louis~Alexandre Rouault.
\newblock Distributed momentum for byzantine-resilient stochastic gradient descent.
\newblock In \emph{9th International Conference on Learning Representations (ICLR)}, number CONF, 2021.

\bibitem[eon Bottou(1998)]{eon1998online}
L~eon Bottou.
\newblock Online learning and stochastic approximations.
\newblock \emph{Online learning in neural networks}, 17\penalty0 (9):\penalty0 142, 1998.

\bibitem[Fang et~al.(2022)Fang, Liu, Gong, and Bentley]{fang2022aflguard}
Minghong Fang, Jia Liu, Neil~Zhenqiang Gong, and Elizabeth~S Bentley.
\newblock Aflguard: Byzantine-robust asynchronous federated learning.
\newblock In \emph{Proceedings of the 38th Annual Computer Security Applications Conference}, pages 632--646, 2022.

\bibitem[Farhadkhani et~al.(2022)Farhadkhani, Guerraoui, Gupta, Pinot, and Stephan]{farhadkhani2022byzantine}
Sadegh Farhadkhani, Rachid Guerraoui, Nirupam Gupta, Rafael Pinot, and John Stephan.
\newblock Byzantine machine learning made easy by resilient averaging of momentums.
\newblock In \emph{International Conference on Machine Learning}, pages 6246--6283. PMLR, 2022.

\bibitem[Guerraoui et~al.(2018)Guerraoui, Rouault, et~al.]{guerraoui2018hidden}
Rachid Guerraoui, S{\'e}bastien Rouault, et~al.
\newblock The hidden vulnerability of distributed learning in byzantium.
\newblock In \emph{International Conference on Machine Learning}, pages 3521--3530. PMLR, 2018.

\bibitem[Guerraoui et~al.(2023)Guerraoui, Gupta, and Pinot]{guerraoui2023byzantine}
Rachid Guerraoui, Nirupam Gupta, and Rafael Pinot.
\newblock Byzantine machine learning: A primer.
\newblock \emph{ACM Computing Surveys}, 2023.

\bibitem[Hazan et~al.(2016)]{hazan2016introduction}
Elad Hazan et~al.
\newblock Introduction to online convex optimization.
\newblock \emph{Foundations and Trends{\textregistered} in Optimization}, 2\penalty0 (3-4):\penalty0 157--325, 2016.

\bibitem[Karimireddy et~al.(2020)Karimireddy, He, and Jaggi]{karimireddy2020byzantine}
Sai~Praneeth Karimireddy, Lie He, and Martin Jaggi.
\newblock Byzantine-robust learning on heterogeneous datasets via bucketing.
\newblock \emph{arXiv preprint arXiv:2006.09365}, 2020.

\bibitem[Karimireddy et~al.(2021)Karimireddy, He, and Jaggi]{karimireddy2021learning}
Sai~Praneeth Karimireddy, Lie He, and Martin Jaggi.
\newblock Learning from history for byzantine robust optimization.
\newblock In \emph{International Conference on Machine Learning}, pages 5311--5319. PMLR, 2021.

\bibitem[Kavis et~al.(2019)Kavis, Levy, Bach, and Cevher]{kavis2019unixgrad}
Ali Kavis, Kfir~Y Levy, Francis Bach, and Volkan Cevher.
\newblock Unixgrad: A universal, adaptive algorithm with optimal guarantees for constrained optimization.
\newblock \emph{Advances in neural information processing systems}, 32, 2019.

\bibitem[Krizhevsky et~al.(2014)Krizhevsky, Nair, and Hinton]{krizhevsky2014cifar}
Alex Krizhevsky, Vinod Nair, and Geoffrey Hinton.
\newblock The cifar-10 dataset.
\newblock \emph{online: \url{https://www.cs.toronto.edu/~kriz/cifar.html}}, 55\penalty0 (5), 2014.

\bibitem[Lamport et~al.(2019)Lamport, Shostak, and Pease]{lamport2019byzantine}
Leslie Lamport, Robert Shostak, and Marshall Pease.
\newblock The byzantine generals problem.
\newblock In \emph{Concurrency: the works of leslie lamport}, pages 203--226. 2019.

\bibitem[Langford et~al.(2009)Langford, Smola, and Zinkevich]{langford2009slow}
John Langford, Alexander Smola, and Martin Zinkevich.
\newblock Slow learners are fast.
\newblock \emph{arXiv preprint arXiv:0911.0491}, 2009.

\bibitem[LeCun et~al.(2010)LeCun, Cortes, Burges, et~al.]{lecun2010mnist}
Yann LeCun, Corinna Cortes, Chris Burges, et~al.
\newblock Mnist handwritten digit database, 2010.
\newblock URL \url{http://yann.lecun.com/exdb/mnist/}.
\newblock Licensed under CC BY-SA 3.0, available at \url{https://creativecommons.org/licenses/by-sa/3.0/}.

\bibitem[Levy(2023)]{levy2023mu}
Kfir~Y Levy.
\newblock $\mu^2$-sgd: Stable stochastic optimization via a double momentum mechanism.
\newblock \emph{arXiv preprint arXiv:2304.04172}, 2023.

\bibitem[Mishchenko et~al.(2022)Mishchenko, Bach, Even, and Woodworth]{mishchenko2022asynchronous}
Konstantin Mishchenko, Francis Bach, Mathieu Even, and Blake~E Woodworth.
\newblock Asynchronous sgd beats minibatch sgd under arbitrary delays.
\newblock \emph{Advances in Neural Information Processing Systems}, 35:\penalty0 420--433, 2022.

\bibitem[Polyak(1964)]{polyak1964some}
Boris~T Polyak.
\newblock Some methods of speeding up the convergence of iteration methods.
\newblock \emph{Ussr computational mathematics and mathematical physics}, 4\penalty0 (5):\penalty0 1--17, 1964.

\bibitem[Stich and Karimireddy(2019)]{stich2019error}
Sebastian~U Stich and Sai~Praneeth Karimireddy.
\newblock The error-feedback framework: Better rates for sgd with delayed gradients and compressed communication.
\newblock \emph{arXiv preprint arXiv:1909.05350}, 2019.

\bibitem[Xie et~al.(2020{\natexlab{a}})Xie, Koyejo, and Gupta]{xie2020fall}
Cong Xie, Oluwasanmi Koyejo, and Indranil Gupta.
\newblock Fall of empires: Breaking byzantine-tolerant sgd by inner product manipulation.
\newblock In \emph{Uncertainty in Artificial Intelligence}, pages 261--270. PMLR, 2020{\natexlab{a}}.

\bibitem[Xie et~al.(2020{\natexlab{b}})Xie, Koyejo, and Gupta]{xie2020zeno++}
Cong Xie, Sanmi Koyejo, and Indranil Gupta.
\newblock Zeno++: Robust fully asynchronous sgd.
\newblock In \emph{International Conference on Machine Learning}, pages 10495--10503. PMLR, 2020{\natexlab{b}}.

\bibitem[Yang and Li(2021)]{yang2021basgd}
Yi-Rui Yang and Wu-Jun Li.
\newblock Basgd: Buffered asynchronous sgd for byzantine learning.
\newblock In \emph{International Conference on Machine Learning}, pages 11751--11761. PMLR, 2021.

\bibitem[Yang and Li(2023)]{yang2023buffered}
Yi-Rui Yang and Wu-Jun Li.
\newblock Buffered asynchronous sgd for byzantine learning.
\newblock \emph{Journal of Machine Learning Research}, 24\penalty0 (204):\penalty0 1--62, 2023.

\bibitem[Yin et~al.(2018)Yin, Chen, Kannan, and Bartlett]{yin2018byzantine}
Dong Yin, Yudong Chen, Ramchandran Kannan, and Peter Bartlett.
\newblock Byzantine-robust distributed learning: Towards optimal statistical rates.
\newblock In \emph{International Conference on Machine Learning}, pages 5650--5659. PMLR, 2018.

\bibitem[Zhao et~al.(2023)Zhao, Zhou, Li, Tang, Wang, Hou, Min, Zhang, Zhang, Dong, et~al.]{zhao2023survey}
Wayne~Xin Zhao, Kun Zhou, Junyi Li, Tianyi Tang, Xiaolei Wang, Yupeng Hou, Yingqian Min, Beichen Zhang, Junjie Zhang, Zican Dong, et~al.
\newblock A survey of large language models.
\newblock \emph{arXiv preprint arXiv:2303.18223}, 2023.

\bibitem[Zhu et~al.(2023)Zhu, Huang, Zhao, and Xu]{zhu2023asynchronous}
Zehan Zhu, Yan Huang, Chengcheng Zhao, and Jinming Xu.
\newblock Asynchronous byzantine-robust stochastic aggregation with variance reduction for distributed learning.
\newblock In \emph{2023 62nd IEEE Conference on Decision and Control (CDC)}, pages 151--158. IEEE, 2023.

\end{thebibliography}

\clearpage

\appendix

\section{Bounded Smoothness Variance Assumption}

\label{sec:sigmal}
We show that Eq.~\eqref{eq:Main} implies that Eq.~\eqref{eq:sigmal} holds for some $\sigmal^2 \in[0,L^2]$.
\begin{align*}
\E\|(\nabla f(\bx;\bz)-\nabla f(\bx)) - (\nabla f(\by;\bz)-\nabla f(\by))\|^2 &= 
\E\|\nabla f(\bx;\bz)-\nabla f(\by;\bz)\|^2 - \|\nabla f(\bx)-\nabla f(\by))\|^2 \\
&\leq L^2 \|\bx-\by\|^2~.  
\end{align*} 
Here, we also used  $\E[\nabla f(\bx;\bz)-\nabla f(\by;\bz)]=(\nabla f(\bx)-\nabla f(\by))$, and followed Eq.~\eqref{eq:Main}. Therefore, we establish that  $\sigmal^2\in[0,L^2]$.

 \section{Asynchronous Robust Convex Analysis}

\subsection{Proof of Thm.~\ref{thm:MainAsy}}
\begin{proof}[Proof of Thm.~\ref{thm:MainAsy}] 
\label{app:main}
To simplify the discussion, let's introduce some notations for a worker $i \in{\GGG}$, who arrives at time $t$:
\begin{gather*}
    \Tilde{\bx}_{s_t} := \bx_{t-\tau_t}= {\bx}^{(i)}_{t}, \quad  \Tilde{\bx}_{s_t - 1} := \bx_{t-\tau_t-\tau_{t-\tau_t}} = {\bx}^{(i)}_{t-\tau_t}\\
    \Tilde{\varepsilon}_{s_t} := \varepsilon_{t-\tau_t}=\varepsilon^{(i)}_{t}, \quad  \Tilde{\varepsilon}_{s_t - 1} := \varepsilon_{t-\tau_t-\tau_{t-\tau_t}}=\varepsilon^{(i)}_{t-\tau_t}
    \\
    \bh_{s_t} := \bg_{t-\tau_t}=\bg_{t}^{(i)}, \quad \Tilde{\bh}_{s_t-1} := \Tilde{\bg}_{t-\tau_t-\tau_{t-\tau_t}} = \Tilde{\bg}^{(i)}_{t-\tau_t}
\end{gather*}
Next, we will employ the following lemma that bounds the distance between the averages $\bx_t, \bx_{t-\tau_t}$. Recall that $\bx_t$ and $\bx_{t-\tau_t}$ are consecutive query points for the worker $i$ that arrives at time $t$. 
\begin{lemma}[\cite{aviv2021asynchronous}]
\label{lem:PointDistworker}
Let $f:\K\mapsto\real$, where $\K$ is a convex set with bounded diameter $D$. 
Then invoking Alg.~\ref{alg:Asy} with $\{\alpha_t = t\}_t$ ensures the following for any $t\in[T]$,
\begin{gather*}
    \|{\bx_{t}-\bx_{t-\tau_t}}\| \leq \frac{4D\tau_{t}}{{t}}~.
\end{gather*}
\end{lemma}
For completeness, we provide a proof in Section~\ref{sec:Proof_lem:PointDistworker}.

Next, we define  $\mu_t$ be the average delay of the worker $i$ that arrives at iteration $t$, i.e., 
\begin{gather*}
    s_t={\frac{t}{\mu_{t}}}, \quad
    s_t - 1 = s_{t-\tau_t}={\frac{t-\tau_t}{\mu_{t-\tau_t}}}~.
\end{gather*}

From Equation \eqref{eq:delay}, we infer that,
    \begin{equation}
    \label{eq:delayAverage}
         \tau^{(i)}_{min} \leq \mu_t  \leq K \tau^{(i)}_{min}~. 
    \end{equation}

 Following this, we analyze the upper bound on the distance between two successive query points for an honest worker $i$ that arrives at time $t$:
\begin{equation}
\label{eq:anytime_delay}
    \|\Tilde{\bx}_{s_t} - \Tilde{\bx}_{s_t-1}\| =\|\bx_{t-\tau_t} - \bx_{t-\tau_t-\tau_{t-\tau_t}}\| \leq \frac{4\tau_{t-\tau_t}}{t-\tau_t} D = \frac{4\tau_{t-\tau_t}}{(\mu_{t-\tau_t})(s_{t}-1)} D \leq \frac{4K}{s_{t}-1} D~,
\end{equation}
where the first inequality follows Lemma \ref{lem:PointDistworker}. The second equality utilizes the relation $s_t - 1 = {\frac{t-\tau_t}{\mu_{t-\tau_t}}}$. The final inequality stems from the assumptions in Eq. \eqref{eq:delay} and Eq. \eqref{eq:delayAverage}. 

\textbf{Remark:} Before proceeding with the analysis, we shall condition the (possible randomization) in the delays of all workers; and recall that the data-samples are independent of the delays.  Thus, the expectations that we take are only with respect to the randomization in the data-samples and are conditioned on the delays. 
Thus, this conditioning allows us to treat the delays $\tau_t^{(i)}$'s and number of updates $s_t^{(i)}$'s as fixed and predefined.

We proceed to analyze the recursive dynamics of $\Tilde{\varepsilon}_{s_t}$ for each $i\in {\GGG}$. Based on the definitions of $\bd_t$ and $\varepsilon_t$, we can present the recursive relationship in the following way:
\begin{align*}
\Tilde{\varepsilon}_{s_t}=\beta_{t}(\bh_{s_t}-\nabla f(\Tilde{\bx}_{s_t}))+(1-\beta_{t})Z_{s_t}+(1-\beta_t)\Tilde{\varepsilon}_{s_t - 1}~,
\end{align*}
where $Z_{s_t}:=\bh_{s_t}-\nabla f(\Tilde{\bx}_{s_t})-(\Tilde{\bh}_{s_t-1}-\nabla f(\Tilde{\bx}_{s_t - 1}))$. Upon choosing $\beta_t = \frac{1}{s_t}$, we can reformulate the above equation as follows: 
\begin{gather*}
    s_t\Tilde{\varepsilon}_{s_t}=(\bh_{s_t}-\nabla f(\Tilde{\bx}_{s_t} ))+(s_t-1)Z_{s_t}+(s_t-1)\Tilde{\varepsilon}_{s_t - 1}~.
\end{gather*}
Unrolling this recursion yields an explicit expression for any $s_t\geq 1$:
\begin{gather}
\label{eq:MartDecomp}
s_t\Tilde{\varepsilon}_{s_t}=\sum_{k\in\left[s_t\right]}\MMM^{(i)}_k~,
\end{gather}
where we have defined,
\begin{gather}
\label{eq:MartDiff_Deff}
\MMM^{(i)}_{k}:=\bh_{k}-\nabla f(\Tilde{\bx}_{k})+(k-1)Z_{k}~,
\end{gather}
and $k$ is a counter for the iterations where worker $i$ makes an update.

Following this, we derive an upper bound for the expected square norm of $\MMM^{(i)}_{k}$ as follows:
\begin{align}
\label{eq:MartBound}
    \E\|{\MMM^{(i)}_k}\|^2 &\leq 2\E\|{\bh_{k}-\nabla f(\Tilde{\bx}_k)}\|^2 + 2(k-1)^2\E\|{(\bh_{k}-\nabla f(\Tilde{\bx}_k))-(\Tilde{\bh}_{k-1}-\nabla f(\Tilde{\bx}_{k-1}))}\|^2 \nonumber\\
    &\leq 2\sigma^2 + 2\sigma_L^2(k-1)^2 \E\|\Tilde{\bx}_k-\Tilde{\bx}_{k-1}\|^2 \nonumber\\
    &\leq 2\sigma^2 + 32D^2K^2 \sigma_L^2 = \tsigma^2~,
\end{align}
where the first inequality uses $\|\ba+\bb\|^2\leq 2\|\ba\|^2+2\|\bb\|^2$, which holds $\forall \ba,\bb\in\real^d$. The second inequality aligns with the assumptions outlined in Equations \eqref{eq:bounded-variance} and \eqref{eq:sigmal}. The third inequality uses Eq.~\eqref{eq:anytime_delay}.  

\textbf{Establishing the First Part of the Theorem:}
Before continuing, it is natural to define an ordered set of samples $\{\bz_1,\bz_2,\ldots, \bz_{t_\GGG}\}$ such that these samples are associated with honest and consecutive updates (or iterates) of the $\PS$, and $t_\GGG$ is the total number of honest updates up to time $t$. Concretely, the $\tau^{\rm{th}}$ honest update of the $\PS$ is based on an honest worker that utilizes a fresh sample $\bz_\tau$.

Now, for a given worker $i$ we shall define the filtration associated with his updates. Concretely, let $k \in \{1,\ldots s_T^{(i)}\}$. Then we define  $\F_k^{(i)}$ to be the sigma-algebra induces by the sequence of samples 
$\{\bz_1,\bz_2,\ldots, \bz_{t_\GGG}\}$ up to the $k^{\rm{th}}$ update of worker $i$. And it is easy to see that $\{\F_k^{(i)}\}_{k\in[s_t^{(i)}]}$ is a filtration. Moreover, it can be directly shown that for a given worker $i$, then the above defined sequence $\{ \MMM_k^{(i)}\}_{k\in[s_t^{(i)}]}$ (see Eq.~\eqref{eq:MartDiff_Deff}) is a martingale difference sequence with respect to $\{\F_k^{(i)}\}_{k\in[s_t^{(i)}]}$.
This allows us to directly employ  Lemma~\ref{lem:SumMart} below, which yields,
\begin{gather*}
    \E\left\|{s^{(i)}_t\varepsilon_{t}^{(i)}} \right\|^2 =\E\left\|\sum_{k\in\left[s^{(i)}_t\right]}\MMM_k^{(i)}\right\|^2  =\sum_{k\in\left[s^{(i)}_t\right]}\E\left\|\MMM_k^{(i)}\right\|^2  \leq \tsigma^2 s_t^{(i)}~, 
\end{gather*}
where we have also used Equations~\eqref{eq:MartDecomp} and \eqref{eq:MartBound}. Thus, the above bounds establish the first part of the theorem.
\begin{lemma}[See e.g.~Lemma B.1 in \cite{levy2023mu}]
\label{lem:SumMart}
Let  $\{ M_t\}_t$ be a martingale difference sequence with respect to a filtration $\{\F_t\}_t$, then the following holds for any $t$,
\begin{align*}
\E \left\|\sum_{\tau\in[t]} M_\tau \right\|^2 
&=
\sum_{\tau\in[t]} \E\left\| M_\tau\right\|^2 ~.
\end{align*}
\end{lemma}
\textbf{Establishing the Second Part of the Theorem:}
As before, we define an ordered set of samples $\{\bz_1,\bz_2,\ldots, \bz_{t_\GGG}\}$ such that these samples are associated with honest and consecutive updates (or iterates) of the $\PS$, and $t_\GGG$ is the total number of honest updates up to time $t$. Concretely, the $\tau^{\rm{th}}$ honest update of the $\PS$ is based on an honest worker that utilizes a fresh sample $\bz_\tau$. We shall also define $\{\F_\tau\}_{\tau\in[t_\GGG]}$ be the natural filtration induced by the ordered sequence of data samples.

Moreover, for a given sample $\bz_\tau\in \{\bz_1,\bz_2,\ldots, \bz_{t_\GGG}\}$, let $i_\tau\in[m]$ be the worker that is associated with the $\tau^{\rm{th}}$ honest update of the $\PS$ 
, with a fresh sample $\bz_\tau$. 
In this case, we shall define:
\begin{align*}
\MMM_\tau : = \MMM_{s_\tau^{(i_\tau)}}^{(i_\tau)}~.
\end{align*}
where $\MMM_k^{(i)}$ is defined in Eq.~\eqref{eq:MartDiff_Deff}.
It is immediate to show that $\{\MMM_\tau\}_{\tau\in[t_\GGG]}$ is a martingale difference sequence with respect to $\{\F_\tau\}_{\tau\in[t_\GGG]}$. Moreover, the following holds directly be the definition of $\MMM_\tau$:
\begin{align}
\label{eq:MartSum1}
\sum_{i\in {\GGG}}\sum_{k\in\left[s_t^{(i)}\right]}\MMM^{(i)}_k = \sum_{\tau=1}^{t_\GGG}\MMM_\tau~.
\end{align}
Now using Eq.~\eqref{eq:MartDecomp} the following holds,
\begin{align}
\label{eq:MartSum2}
\sum_{i\in {\GGG}}s^{(i)}_t\varepsilon_t^{(i)} = \sum_{i\in {\GGG}}\sum_{k\in\left[s_t^{(i)}\right]}\MMM^{(i)}_k~.
\end{align}
Combining Equations~\eqref{eq:MartSum1} and \eqref{eq:MartSum2} together with Lemma~\ref{lem:SumMart} establishes the second part of the theorem,
\begin{gather*}
    \E\left\|{\sum_{i\in {\GGG}}s^{(i)}_t\varepsilon_t^{(i)}} \right\|^2
    = 
    \E\left\| \sum_{i\in {\GGG}}\sum_{k\in\left[s_t^{(i)}\right]}\MMM^{(i)}_k\right\|^2 
   = 
    \E\left\| \sum_{\tau=1}^{t_\GGG}\MMM_\tau\right\|^2 
   =
    \sum_{\tau=1}^{t_\GGG}\E\left\|\MMM_\tau\right\|^2
    \leq  \tsigma^2t_{\GGG}~.
\end{gather*}
where the inequality uses the bound in Eq.~\eqref{eq:MartBound}.
\end{proof}
\subsubsection{Proof of Lemma~\ref{lem:PointDistworker}}
\label{sec:Proof_lem:PointDistworker}
\begin{proof}
    We borrowed the following steps from \cite{aviv2021asynchronous}. Let's define $\by\in\K$ as the average of $\bx_t$ over the interval $[t - \tau_t, t]$, i.e.,
    \begin{align*}
        \by := \frac{1}{\alpha_{t-\tau_t+1:t}} \sum_{i=t-\tau_t+1}^{t} \alpha_i \bw_i~.
    \end{align*}
    Then we have the following relationship:
\begin{align*}
    \alpha_{1:t} \bx_t = \sum_{i=1}^{t} \alpha_i \bw_i 
    = \sum_{i=1}^{t-\tau_t} \alpha_i \bw_i + \sum_{i=t-\tau_t+1}^{t} \alpha_i \bw_i
    = \alpha_{1:t-\tau_t} \bx_{t-\tau_t} + \alpha_{t-\tau_t+1:t} \by~. 
\end{align*}
Hence,
\begin{align*}
    \alpha_{1:t-\tau_t} (\bx_t - \bx_{t-\tau_t}) = \alpha_{t-\tau_t+1:t} (\by - \bx_t)~.
\end{align*}
By setting \( \alpha_t = t \), we have that,
\begin{align*}
    \|\bx_t - \bx_{t-\tau_t}\| &= \frac{\alpha_{t-\tau_t+1:t}}{\alpha_{1:t-\tau_t}} \|\by - \bx_t\| \\
    & = \frac{\tau_t (t - \tau_t + 1 + t)}{(t - \tau_t)(t - \tau_t + 1)} \|\by - \bx_t\| \\ 
    & \leq
    \frac{\tau_t (t-\tau_t + 1)}{(t - \tau_t)(t - \tau_t + 1)} \|\by - \bx_t\| + \frac{t\tau_t}{(t - \tau_t)^2} \|\by - \bx_t\|\\ 
    & =
    \frac{\tau_t}{t - \tau_t} \|\by - \bx_t\|+ \frac{t\tau_t}{(t - \tau_t)^2} \|\by - \bx_t\|~. 
\end{align*}
For \( t \geq 3 \tau_t \), we have that,
\begin{align*}
    \|\bx_t - \bx_{t-\tau_t}\| \leq \frac{3\tau_t D}{2t}+\frac{9\tau_t D}{4t}\leq\frac{4\tau_t D}{t}~.
\end{align*}
Given that the domain is bounded, \( \|\bx_t - \bx_{t-\tau_t}\| \leq D \; \forall t \), for \( t < 3 \tau_t \), we have \( D < \frac{4 \tau_t D}{t} \). Combining these results, we conclude:
\[
\|\bx_t - \bx_{t-\tau_t}\| \leq \frac{4 \tau_t D}{t}~.
\]
\end{proof}
\subsection{Proof of Lemma~\ref{lem:asyncFilter}}

\begin{proof} [Proof of Lemma~\ref{lem:asyncFilter}]

\begin{lemma}
\label{lem:PointDist}
Let $f:\K\mapsto\real$, where $\K$ is a convex set with bounded diameter $D$. 
Then invoking Alg.~\ref{alg:Asy} with $\{\alpha_t = t\}_t$ ensures the following for any $t\in[T]$, and every $i, j\in [m]$,
\begin{gather*}
   \left\|{\bx_{t}^{(i)}-\bx_{t}^{(j)}}\right\| \leq \frac{4D\left(\tau_{t}^{(i)}+ \tau_{t}^{(j)}\right)}{{t}}~.
\end{gather*}
 
\begin{proof} 
\begin{gather*}
   \left\|{\bx_{t}^{(i)}-\bx_{t}^{(j)}}\right\| \leq 
   \left\|{\bx_{t}^{(i)}-\bx_{t}}\right\| + \left\|{\bx_{t}-\bx_{t}^{(j)}}\right\|\leq \frac{4D\tau_{t}^{(i)}}{{t}} +\frac{4D\tau_{t}^{(j)}}{{t}}~,
\end{gather*}
where the first inequality is a result of the triangle inequality, and the second follows Lemma \ref{lem:PointDistworker}.
\end{proof}
\end{lemma}

\paragraph{Bias Bounds.}
We begin by analyzing the upper bound of the bias in the collective gradients of honest workers up to time \( t \) in relation to the gradient at that time, denoted as \( \BBB^1_t \). Following this, we derive the upper bound for the bias between the collective gradients of these honest workers and the gradient of an individual honest worker, also up to time \( t \), which we denote as \( \BBB^2_t \). For clarity, we define $\Bar{\nabla}_{\GGG,t}:=\frac{1}{\sum_{i\in {\GGG}} s^{(i)}_t} \sum_{i\in {\GGG}} s^{(i)}_t \nabla f(\bx^{(i)}_{t})$.
\begin{align}
\left\|{\BBB^1_t}\right\| &:= \E\left\|\Bar{\nabla}_{\GGG,t} -\nabla f(\bx_t)\right\| \nonumber
\\ & =\E\left\|{\frac{1}{\sum_{i\in {\GGG}} s^{(i)}_t} \sum_{i\in {\GGG}} s^{(i)}_t \nabla f(\bx^{(i)}_{t}) -\nabla f(\bx_t)}\right\|  \nonumber
\\ &
\leq \E\left[\frac{1}{\sum_{i\in {\GGG}} s^{(i)}_t} \sum_{i\in {\GGG}} s^{(i)}_t \left\|{\nabla f(\bx^{(i)}_{t}) -\nabla f(\bx_t)} \right\|\right]  \nonumber
\\ & \leq \frac{L}{\sum_{i\in {\GGG}} s^{(i)}_t}\E\left[\sum_{i\in {\GGG}} s^{(i)}_t \left\|{\bx^{(i)}_{t} -\bx_t} \right\| \right] \nonumber
\\ &\leq \frac{4DL}{\sum_{i\in {\GGG}} s^{(i)}_t} \E \left[\sum_{i\in {\GGG}} s^{(i)}_t \frac{\tau_{t}^{(i)}}{t}\right] \nonumber
\\ &  \leq
\frac{4\tau^{max}_{t}DL}{t}~. \label{eq:bias1}
\end{align}
Here, the first inequality leverages Jensen's inequality, and the second follows the smoothness assumption in Eq. \eqref{eq:Main}. The third follows Lemma \ref{lem:PointDistworker}, and the last inequality follows that $\tau_t^{max}:=\max_{i\in[m]} \{\tau_t^{(i)}\}$.

For the second bias $\BBB^2_t$, we have:
\begin{align}
\E\left\|{\BBB^2_t}\right\|&:=\E\left\|{\nabla f(\bx_{t}^{(i)}) - \Bar{\nabla}_{\GGG,t}}\right\| \nonumber
\\ & = \E\left\|{\nabla f(\bx_{t}^{(i)}) - \frac{1}{\sum_{j\in {\GGG}} s^{(j)}_t} \sum_{j\in {\GGG}} s^{(j)}_t \nabla f(\bx_{t}^{(j)})}\right\| \nonumber
\\ & \leq \frac{1}{\sum_{j\in {\GGG}} s^{(j)}_t}  \E \left[\sum_{j\in {\GGG}} s^{(j)}_t \left\|{\nabla f(\bx_{t}^{(i)}) -\nabla f(\bx_{t}^{(j)})} \right\| \right] \nonumber
\\ & \leq \frac{L}{\sum_{j\in {\GGG}} s^{(j)}_t}  \E \left[\sum_{j\in {\GGG}}s^{(j)}_t\left\|{\bx_{t}^{(i)} - \bx_{t}^{(j)}} \right\|\right] \nonumber
\\ & \leq \frac{4DL}{\sum_{j\in {\GGG}} s^{(j)}_t}\E\left[\sum_{j\in {\GGG}}  s^{(j)}_t \left( \frac{\tau_{t}^{(i)} + \tau_{t}^{(j)}}{t} \right)\right] \nonumber
\\ & \leq
\frac{8\tau^{max}_{t}DL}{t}~. \label{eq:bias2}
\end{align}
Like before, the first inequality leverages Jensen's inequality, and the second follows the smoothness assumption in Eq. \eqref{eq:Main}, the third inequality follows Lemma \ref{lem:PointDist}, and the last one follows that $\tau_t^{max}:=\max_{i\in[m]} \{\tau_t^{(i)}\}$.
    
\paragraph{Variance Bound.} We start by determining $\rho_i$ as outlined in Definition \ref{def2}:
\begin{align*}
    \E \|{\bd_t^{(i)} - \Bar{\bd}_{\GGG,t}}\|^2 &  
    \leq 3\E\|{\bd_t^{(i)} - \nabla f (\bx_t^{(i)})}\|^2+3\E\|{\Bar{\nabla}_{\GGG,t} - \Bar{\bd}_{\GGG,t}}\|^2 + 3\E\left\|{\BBB^2_t}\right\|^2\\ &
    = 3\E \left\|\varepsilon_t^{(i)}\right\|^2+3\E\left\|\frac{1}{\sum_{i\in\GGG}s_t^{(i)}} \sum_{i\in\GGG}s_t^{(i)}\varepsilon_t^{(i)}\right\|^2 + {3\E\left\|{\BBB^2_t}\right\|^2} \\ & \leq \frac{3\tsigma^2}{s_t^{(i)}} + \frac{3\tsigma^2}{t_\GGG} + \frac{192(\tau^{max}_{t}DL)^2}{t^2} \\ & \leq \frac{6\tsigma^2}{s_t^{(i)}} + \frac{192(\tau^{max}_{t}DL)^2}{t^2}~, 
\end{align*}
where $\Bar{\bd}_{\GGG,t}:=\frac{1}{\sum_{i\in {\GGG}} s^{(i)}_t} \sum_{i\in {\GGG}} s^{(i)}_t \bd_t^{(i)}$. The first and inequality uses $\|\ba+\bb+\bc\|^2\leq 3\|\ba\|^2+3\|\bb\|^2+3\|\bc\|^2$, which holds $\forall 
 \ba,\bb, \bc\in\real^d$. The second inequality follows Theorem \ref{thm:MainAsy} and employs the second bias bound in Eq. \eqref{eq:bias2}. The third uses the fact that $s_t^{(i)}\leq t_\GGG, \forall i\in\GGG$. Accordingly, we set $\rho_i^2:=\frac{6\tsigma^2}{s_t^{(i)}} + \frac{192(\tau^{max}_{t}DL)^2}{t^2} $. 

Following this, we derive $\rho$ as outlined in Definition \ref{def2}:
\begin{align}
\label{eq:asy-rho}
    \rho^2 = \frac{1}{\sum_{i\in\GGG}s_{t}^{(i)}}\sum_{i\in\GGG}s_{t}^{(i)} \rho^2_i = \frac{1}{t_\GGG}\sum_{i\in\GGG}s_{t}^{(i)} \left(\frac{6\tsigma^2}{s_t^{(i)}} + \frac{192(\tau^{max}_{t}DL)^2}{t^2}\right) = \frac{6m\tsigma^2}{t_\GGG} + \frac{192(\tau^{max}_{t}DL)^2}{t^2}~.
\end{align}

Next, we establish an upper bound for \( \E\| \EEE_t \|^2 \):
\begin{align*}
    \E\| \EEE_t \|^2 &= \E\left\| \hat{\bd}_t-\nabla f(\bx_t) \right\|^2 \\
    &\leq 2\E\left\|\hat{\bd}_t-\Bar{\bd}_{\GGG,t} \right\|^2 + 2\E\left\| \Bar{\bd}_{\GGG,t}-\nabla f(\bx_t) \right\|^2 \\
    &\leq {2c_\lambda}\left(\frac{6m\tsigma^2}{t_\GGG} + \frac{192(\tau^{max}_{t}DL)^2}{t^2} \right) + 4\E\left\| \Bar{\bd}_{\GGG,t}-\Bar{\nabla}_{\GGG,t} \right\|^2 + 4\E\left\|\BBB_t^1 \right\|^2 \\
    &=  {2c_\lambda}\left(\frac{6m\tsigma^2}{t_\GGG} + \frac{192(\tau^{max}_{t}DL)^2}{t^2} \right) + 4\E\left\| \frac{1}{\sum_{i\in\GGG}s_t^{(i)}} \sum_{i\in\GGG}s_t^{(i)}\varepsilon_t^{(i)} \right\|^2 + 4\E\left\| \BBB_t^1 \right\|^2 \\
    &\leq \frac{12c_\lambda m\tsigma^2}{t_\GGG}+\frac{4\tsigma^2}{t_\GGG} + \frac{(\tau^{max}_{t}DL)^2(384c_\lambda + 64)}{t^2} \\ &\leq \frac{8\tsigma^2}{t} + \frac{24c_\lambda m\tsigma^2}{t}+ \frac{64(\tau^{max}_{t}DL)^2}{t^2} + \frac{384c_\lambda(\tau^{max}_{t}DL)^2}{t^2}~,
\end{align*}
where the first inequality uses $\|\ba+\bb\|^2\leq 2\|\ba\|^2+2\|\bb\|^2$, which holds $\forall\ba,\bb\in\real^d$. The second inequality utilizes the same inequality and is further supported by Definition \ref{def2} and Equation \eqref{eq:asy-rho}. The third aligns with Theorem \ref{thm:MainAsy}, and employs the first bias bound in Eq. \eqref{eq:bias1}. The last one utilizes the fact that $t_\GGG\geq(1-\lambda)t\geq t/2$, given $\lambda < \nicefrac{1}{2}$.
\end{proof}

\subsection{Proof of Thm.~\ref{thm:AsymuSGD}}
\begin{proof}[Proof of Thm.~\ref{thm:AsymuSGD}] 

Following Lemma \ref{lem:asyncFilter} and applying Jensen's inequality, we derive the following bound:
\begin{align}
\label{eq:EEE}
    \E \|\EEE_t\| = \E \sqrt{\|\EEE_t\|^2} \leq \sqrt{\E \|\EEE_t\|^2 } \leq O\left(\frac{\tsigma\sqrt{1+mc_\lambda}}{\sqrt{t}} + \frac{\tau^{max}_{t}DL\sqrt{1+c_\lambda}}{t}\right)~,
\end{align}
 where the third inequality uses \(\sqrt{a+b}\leq \sqrt{a}+\sqrt{b}\) for non-negative \(a, b \in \real\). The explanation behind this can be seen through the following steps:
\begin{align*}
    \left(\sqrt{a}+\sqrt{b}\right)^2 &= a + 2\sqrt{ab} + b \geq a + b~,
\end{align*}
whereby taking the square root of both sides of this equation, we obtain the desired inequality.

Next, let's revisit the AnyTime guarantee as outlined in \cite{cutkosky2019anytime} and proceed to delve into the regret analysis of the update rule.

\begin{theorem}[Rephrased from Theorem 1 in \citet{cutkosky2019anytime}]
\label{theorem1}
    Let $f:\K\rightarrow\real$ be a convex function with a minimum $\bx^*\in\arg\min_{\bw\in\K}f(\bw)$. Also let $\{\alpha_t\geq 0\}_t$, and $\{\bw_t\in\K\}_t$, $\{\bx_t\in\K\}_t$, such that $\{\bx_t\}_t$ is an $\{\alpha_t\}_t$ weighted averaged of $\{\bw_t\}_t$, i.e. such that $\bx_1=\bw_1$, and for any $t\geq 1$,
    \begin{equation*}
        \bx_{t+1}=\frac{1}{\alpha_{1:t+1}}\sum_{\tau\in[t+1]}{\alpha_\tau\bw_\tau}~.
    \end{equation*}
    Then the following holds for any $t\geq 1$:
    \begin{equation*}
        \alpha_{1:t}(f(\bx_t)-f(\bx^*))\leq\sum_{\tau\in[t]}\alpha_\tau\nabla f(\bx_\tau)(\bw_\tau-\bx^*)~.
    \end{equation*}
\end{theorem}

\begin{lemma} \label{lem:RegretBound}
Let $f:\K\rightarrow\real$ be a convex function with a minimum $\bx^*\in\arg\min_{\bw\in\K}f(\bw)$, and assume that the assumption in Eq. \eqref{eq:bounded_diameter} holds. Also let $\{\alpha_t\geq 0\}_t$, and $\{\bw_t\in\K\}_t$. Then, for any $t\geq 1$, an arbitrary vector $\hat{\bd}_t\in\real^d$, and the update rule:
\begin{gather*}
    \bw_{t+1}=\Pi_{\K}\left({\bw_{t}- \eta\alpha_{t}\hat{\bd}_t}\right)~,
\end{gather*}
we have, 
\begin{align*}
    \sum_{\tau=1}^t \alpha_\tau \langle \hat{\bd}_\tau, \bw_{\tau+1}-\bx^*\rangle \leq 
    \frac{D^2}{2\eta} -\frac{1}{2\eta}\sum_{\tau=1}^t \|\bw_{\tau}-\bw_{\tau+1}\|^2~.
\end{align*}
\end{lemma}

\begin{lemma}
\label{lemma2}
let $f:\K \rightarrow \real$ be an $L$-smooth and convex function, and let $\bx^*\in\arg\min_{\bx\in\K}f(\bx) $, then for any $\bx\in\real^d$ we have,
\begin{equation*}
    \|{\nabla f(\bx)}-\nabla f(\bx^*)\|^2 \leq 2L(f(\bx)-f(\bx^*))~.
\end{equation*}
\end{lemma}

Next, for every iteration $t\leq T$, we define:
\begin{gather*}
   \hat{\bd}_t:=\mathcal{A}_\omega(\{\bd_t^{(i)}, s_t^{(i)}\}_{i=1}^m)
    \\
    \EEE_t:=\hat{\bd}_t-\nabla f(\bx_t)
\end{gather*}
Thus, combining Theorem~\ref{theorem1} with Lemma~\ref{lem:RegretBound}, we have that,
\begin{align}
\label{eq:regret-main}
    \alpha_{1:t}(f(\bx_t)-f(\bx^*)) &\leq \sum_{\tau\in[t]} \alpha_\tau\langle \nabla f(\bx_\tau), \bw_\tau-\bx^*\rangle \nonumber \\ 
    &= \sum_{\tau\in[t]}\alpha_\tau\langle \hat{\bd}_\tau,\bw_{\tau+1}-\bx^*\rangle + \sum_{\tau\in[t]}\alpha_\tau\langle \hat{\bd}_\tau,\bw_\tau-\bw_{\tau+1}\rangle - \sum_{\tau\in[t]}\alpha_\tau \langle\EEE_\tau,\bw_\tau-\bx^*\rangle \nonumber \\
    &\leq \frac{D^2}{2\eta} - \frac{1}{2\eta} \sum_{\tau\in[t]} \|{\bw_\tau-\bw_{\tau+1}}\|^2  + \sum_{\tau\in[t]}\alpha_\tau\langle \hat{\bd}_\tau,\bw_\tau-\bw_{\tau+1}\rangle  - \sum_{\tau\in[t]}\alpha_\tau\langle\EEE_\tau,\bw_\tau-\bx^*\rangle \nonumber \\
    &= \frac{D^2}{2\eta} - \frac{1}{2\eta} \sum_{\tau\in[t]} \|{\bw_\tau-\bw_{\tau+1}}\|^2 +  \sum_{\tau\in[t]}\alpha_\tau\langle \nabla f(\bx_\tau),\bw_\tau-\bw_{\tau+1}\rangle  -\sum_{\tau\in[t]}\alpha_\tau \langle\EEE_\tau,\bw_{\tau+1}-\bx^*\rangle \nonumber \\  &\leq \frac{D^2}{2\eta} \underbrace{- \frac{1}{2\eta} \sum_{\tau\in[t]} \|{\bw_\tau-\bw_{\tau+1}}\|^2 +  \sum_{\tau\in[t]}\alpha_\tau\langle \nabla f(\bx_\tau),\bw_\tau-\bw_{\tau+1}\rangle  }_{\textnormal{(A)}}+D\sum_{\tau\in[t]}\alpha_\tau\|{\EEE_\tau}\|~, 
\end{align}
where the first inequality is derived from the Anytime guarantee, as outlined in Theorem \ref{theorem1}. The second inequality follows Lemma \ref{lem:RegretBound}.  The third inequality is a result of applying the Cauchy-Schwarz inequality and the assumption in Eq. \eqref{eq:bounded_diameter}.
\begin{align*}
    \textnormal{(A)} &:=  - \frac{1}{2\eta} \sum_{\tau\in[t]} \|{\bw_\tau-\bw_{\tau+1}}\|^2 +  \sum_{\tau\in[t]}\alpha_\tau\langle \nabla f(\bx_\tau),\bw_\tau-\bw_{\tau+1}\rangle
    \\ & =- \frac{1}{2\eta} \sum_{\tau\in[t]} \|{\bw_\tau-\bw_{\tau+1}}\|^2 +  \sum_{\tau\in[t]}\alpha_\tau\langle \nabla f(\bx_\tau)-\nabla f(\bx^*),\bw_\tau-\bw_{\tau+1}\rangle + \sum_{\tau\in[t]}\alpha_\tau  \langle \nabla f(\bx^*),\bw_\tau-\bw_{\tau+1}\rangle \\ 
    & \leq \frac{\eta}{2} \sum_{\tau\in[t]} \alpha_\tau^2\|{\nabla f(\bx_\tau)-\nabla f(\bx^*)}\|^2 + \sum_{\tau\in[t]}\alpha_\tau\langle \nabla f(\bx^*),\bw_\tau-\bw_{\tau+1}\rangle\\ 
    & \leq 2\eta L \sum_{\tau\in[t]} \alpha_{1:\tau}\Delta_\tau + \sum_{\tau\in[t]}(\alpha_\tau-\alpha_{\tau-1})\langle \nabla f(\bx^*),\bw_\tau\rangle - \alpha_t \langle \nabla f(\bx^*),\bw_{t+1}\rangle \\ 
    & = 2\eta L \sum_{\tau\in[t]} \alpha_{1:\tau}\Delta_\tau + \sum_{\tau\in[t]}(\alpha_\tau-\alpha_{\tau-1})\langle \nabla f(\bx^*),\bw_\tau - \bw_{t+1}\rangle \\
    & \leq 2\eta L \sum_{\tau\in[t]} \alpha_{1:\tau}\Delta_\tau + \sum_{\tau\in[t]}(\alpha_\tau-\alpha_{\tau-1})\| \nabla f(\bx^*)\| \|\bw_\tau - \bw_{t+1}\| \\
    & \leq  \frac{1}{2T} \sum_{\tau\in[T]} \alpha_{1:\tau}\Delta_\tau + \alpha_tG^*D ~.
\end{align*}
Here, the first inequality employs the Young’s inequality.
For the second inequality, we introduce the notation \(\Delta_t := f(\bx_t) - f(\bx^*)\), and we follow Lemma \ref{lemma2}, which relates to the smoothness of the function $f$. In this step, we also set $\alpha_0 = 0$ and utilizes the property \(\alpha_\tau^2 \leq 2\alpha_{1:\tau}\), given that \(\alpha_\tau = \tau\). The third inequality uses the Cauchy-Schwarz inequality. The last inequality follows the assumption in Eq. \eqref{eq:bounded_diameter}. It uses the fact that \(t \leq T\) and $\Delta_t\geq0$, $\forall t$. This step also incorporates the choice of an appropriate learning rate parameter \(\eta\leq1/4LT\).

Plugging \textnormal{(A)} into Eq. \eqref{eq:regret-main}, gives us,
\begin{align}
\label{eq:reg-final}
    \alpha_{1:t}\Delta_t &\leq \frac{1}{2T} \sum_{\tau\in[T]} \alpha_{1:\tau}\Delta_\tau+ \frac{D^2}{2\eta} + \alpha_tG^*D +D\sum_{\tau\in[t]}\alpha_\tau\|{\EEE_\tau}\|~. 
\end{align}

\begin{lemma}[Lemma C.2 in \cite{levy2023mu}]
\label{lemma3}
let $\{A_t\}_{t\in[T]}, \{B_t\}_{t\in[T]}$ be sequences of non-negative elements, and assume that for any $t\leq T$,
\begin{align*}
    A_t \leq B_T + \frac{1}{2T} \sum_{t\in[T]} A_t~.
\end{align*}
Then the following bound holds,
\begin{align*}
    A_T \leq 2B_T~.
\end{align*}
\end{lemma}

In the next step, let us define two terms:  $A_t := \alpha_{1:t}\E\left[f(\bx_t)-f(\bx^*)\right]$ and $B_t := \frac{D^2}{2\eta}+ \alpha_tG^*D  + D\sum_{\tau\in[t]}\alpha_\tau\E\|{\EEE_\tau}\|$. Note that the series $\{B_t\}_t$ forms a non-decreasing series of non-negative values, implying $B_t \leq B_T$ for any $t\in[T]$. As a result of Eq. \eqref{eq:reg-final}, we have that $A_t\leq B_T + \frac{1}{2T}\sum_{\tau\in[T]}A_\tau$. 

Leveraging Lemma \ref{lemma3}, Eq. \eqref{eq:EEE}, and acknowledging that \(\alpha_{1:T}=\Theta(T^2)\), as \(\alpha_t=t\), it follows that:
\begin{align*}
    \E[f(\bx_T)-f(\bx^*)] &\leq \frac{2}{T^2}\B_T \\
    & = \frac{D^2}{T^2\eta}+ \frac{2G^*D}{T}   + \frac{2D }{T^2}\sum_{t\in[T]}\alpha_t\E\|{\EEE}\| \\
    & \leq O\left(\frac{D^2}{T^2\eta}+ \frac{G^*D}{T} + \frac{D }{{T^2}}\sum_{t\in[T]}\left(\sqrt{t}{\tsigma\sqrt{1+mc_\lambda}} + {\tau^{max}_{t}DL\sqrt{1+c_\lambda}}\right)\right)\\
    & \leq O\left(\frac{D^2}{T^2\eta}+ \frac{G^*D}{T} + \frac{D\tsigma\sqrt{1+mc_\lambda}}{\sqrt{T}} + \frac{\mu^{max}D^2L\sqrt{1+c_\lambda}}{{T}}\right)~,
\end{align*}
where $\mu^{max}:=\frac{1}{T}\sum_{t\in[T]}\tau_t^{max}$. Finally, choosing the optimal \(\eta \leq \frac{1}{4TL}\) gives us:
\begin{equation*}
    \E[f(\bx_T)-f(\bx^*)] \leq O\left(\frac{LD^2\mu^{max}\sqrt{1+c_\lambda}}{{T}} + \frac{G^*D}{T} + \frac{D\tsigma\sqrt{1+mc_\lambda}}{\sqrt{T}}\right)~.
\end{equation*} 
\end{proof}

\subsubsection{Proof of Lemma ~\ref{lem:RegretBound}}

\begin{proof}[Proof of Lemma ~\ref{lem:RegretBound}]
The update rule $\bw_{\tau+1} = \Pi_\K (\bw_\tau - \eta\alpha_\tau\hat{\bd}_\tau)$ can be expressed as a convex optimization problem within the set $\K$:
\begin{align*}
    \bw_{\tau + 1} &= \Pi_{\K}\left({\bw_{\tau} - \eta \alpha_{\tau} \hat{\bd}_\tau}\right) 
    \\ &= \arg\min_{\bw\in\K} \|\bw_{\tau} - \eta \alpha_{\tau} \hat{\bd}_\tau - \bw\|^2 
    \\ &= \arg\min_{\bw\in\K}\{ \alpha_\tau \langle \hat{\bd}_\tau, \bw - \bw_\tau\rangle + \frac{1}{2\eta} \|\bw - \bw_\tau\|^2\}~.
\end{align*}

Here, the first equality is derived from the definition of the update rule, the second stems from the property of projection, and the final equality is obtained by reformulating the optimization problem in a way that does not affect the minimum value.

Given that \(\bw_{\tau+1}\) is the optimal solution of the above convex problem, by the optimality conditions, we have that:
    \begin{align*}
        \left\langle \alpha_\tau\hat{\bd}_\tau + \frac{1}{\eta}(\bw_{\tau+1} - \bw_{\tau}), \bw-\bw_{\tau+1}\right\rangle \geq 0, \quad \forall \bw\in\K~.
    \end{align*}

Rearranging this, summing over $t\geq1$ iterations, and taking $\bw=\bx^*$, we derive:
    \begin{align*}
        \sum_{\tau\in[t]}\alpha_\tau \langle \hat{\bd}_\tau, \bw_{\tau+1}-\bx^*\rangle & \leq \frac{1}{\eta} \sum_{\tau\in[t]}\langle \bw_\tau - \bw_{\tau+1}, \bw_{\tau+1}-\bx^*\rangle \\ &=  \frac{1}{2\eta}\sum_{\tau\in[t]} \left( \|\bw_\tau-\bx^*\|^2 - \|\bw_{\tau+1}-\bx^*\|^2 - \|\bw_\tau-\bw_{\tau+1}\|^2\right) \\ &=  \frac{1}{2\eta} \left( \|\bw_1-\bx^*\|^2 - \|\bw_{t+1}-\bx^*\|^2 - \sum_{\tau\in[t]}\|\bw_\tau-\bw_{\tau+1}\|^2\right) \\ & \leq \frac{D^2}{2\eta} -\frac{1}{2\eta}\sum_{\tau\in[t]} \|\bw_{\tau}-\bw_{\tau+1}\|^2~.
    \end{align*}
The first equality equality is achieved through algebraic manipulation, and the last inequality follows the assumption in Eq. \eqref{eq:bounded_diameter}. 
\end{proof}

\subsubsection{Proof of Lemma ~\ref{lemma2}}
\begin{proof}[Proof of Lemma ~\ref{lemma2}]

\begin{lemma}[Lemma C.1 in \citet{levy2023mu}]
\label{lem:smooth-functions}
let $f:\real^d \rightarrow \real$ be an $L$-smooth function with a global minimum $\bx^*$, then for any $\bx\in\real^d$ we have,
\begin{equation*}
    \|{\nabla f(\bx)}\|^2 \leq 2L(f(\bx)-f(\bx^*))~.
\end{equation*}
\end{lemma}

     Let us define the function $h(\bx)=f(\bx)-f(\bx^*)-\langle \nabla f(\bx^*), \bx - \bx^* \rangle$. Due to the convexity of \( f(\bx) \), we have the gradient inequality \( f(\bx) - f(\bx^*) \geq \langle \nabla f(\bx^*), \bx - \bx^* \rangle \), which implies \( h(\bx) \geq 0 \). As $h(\bx^*)=0$, this implies that $\bx^*$ is the global minimum of $h$. Applying Lemma \ref{lem:smooth-functions}, gives us,
    \begin{align*}
        \|{\nabla f(\bx)}-\nabla f(\bx^*)\|^2 = \|{\nabla h(\bx)}\|^2 \leq 2L(f(\bx)-f(\bx^*)-\langle \nabla f(\bx^*), \bx - \bx^* \rangle) \leq 2L(f(\bx)-f(\bx^*))~,
    \end{align*}
where last inequality holds due to the convexity of \( f \), which implies that \( \langle \nabla f(\bx^*), \bx - \bx^* \rangle \geq 0 \).
\end{proof}

\section{Robust Aggregators Analysis}
\label{app:robust-agg}

\subsection{Weighted Robust Aggregators}

\subsubsection{Weighted Geometric Median (WeightedGM)}
The Weighted Geometric Median (WeightedGM) is an aggregation method that seeks a point minimizing the weighted sum of Euclidean distances to a set of points. Formally, for a given set of points $\{\bx_i\}_{i=1}^m$ and corresponding weights $\{s_i\}_{i=1}^m$, the WeightedGM aggregator is defined as follows:
\[
\textnormal{WeightedGM} \in \arg \min_{\by \in \real^d} \sum_{i\in[m]} s_i \|\by - \bx_i\|
\]

\begin{lemma}
\label{lem:WeightedGM}
    Let ${\hat{\bx}}$ be a WeightedGM aggregator then $\hat{\bx}$ is $(c_\lambda, \lambda)$-weighted robust with $c_\lambda=\left(1+\frac{\lambda}{1-2\lambda}\right)^2$.

\begin{proof}
\begin{align*}
    \|\hat{\bx} - \Bar{\bx}_\GGG \|&=\left\|\hat{\bx} - \frac{1}{\sum_{i\in\GGG} s_i}\sum_{i\in\GGG}s_i{\bx}_i \right\| \\ &
    \leq \frac{1}{\sum_{i\in\GGG} s_i}\sum_{i\in\GGG}s_i \left\|\hat{\bx} -{\bx}_i \right\|\\ &
    = \frac{1}{\sum_{i\in\GGG} s_i}\sum_{i\in[m]}s_i \left\|\hat{\bx} -{\bx}_i \right\| - \frac{1}{\sum_{i\in\GGG} s_i}\sum_{i\in\BBB}s_i \left\|\hat{\bx} -{\bx}_i \right\| \\ &
    \leq \frac{1}{\sum_{i\in\GGG} s_i}\sum_{i\in[m]}s_i \left\|\Bar{\bx}_\GGG -{\bx}_i \right\| - \frac{1}{\sum_{i\in\GGG} s_i}\sum_{i\in\BBB}s_i \left\|\hat{\bx} -{\bx}_i \right\| \\ &
    = \frac{1}{\sum_{i\in\GGG} s_i}\sum_{i\in\GGG}s_i \left\|\Bar{\bx}_\GGG -{\bx}_i \right\| + \frac{1}{\sum_{i\in\GGG} s_i}\sum_{i\in\BBB}s_i \left\|\Bar{\bx}_\GGG -{\bx}_i \right\| - \frac{1}{\sum_{i\in\GGG} s_i}\sum_{i\in\BBB}s_i \left\|\hat{\bx} -{\bx}_i \right\| \\ &
    \leq \frac{1}{\sum_{i\in\GGG} s_i}\sum_{i\in\GGG}s_i \left\|\Bar{\bx}_\GGG -{\bx}_i \right\| + \frac{1}{\sum_{i\in\GGG} s_i}\sum_{i\in\BBB}s_i \left\|\Bar{\bx}_\GGG -\hat{\bx} \right\| \\ &
    \leq \frac{1}{\sum_{i\in\GGG} s_i}\sum_{i\in\GGG}s_i \left\|\Bar{\bx}_\GGG -{\bx}_i \right\| + \frac{\lambda}{1-\lambda} \left\|\Bar{\bx}_\GGG -\hat{\bx} \right\|~.
\end{align*}
The first inequality leverages Jensen's inequality. The second inequality follows the WeightedGM definition. The third is derived using the following triangle inequality: $\|\bar{\bx}_\GGG-\bx_i\|\leq\|\bar{\bx}_\GGG-\hat{\bx}\|+\|\hat{\bx}-\bx_i\|$. The final inequality is based on the assumptions that $\sum_{i\in\BBB}s_i \leq \lambda s_{1:m}$ and $\sum_{i\in\GGG}s_i \geq (1-\lambda) s_{1:m}$.

By rearranging, we obtain:
\begin{align*}
    \|\hat{\bx} - \Bar{\bx}_\GGG \|&
    \leq \left(1+\frac{\lambda}{1-2\lambda}\right)\frac{1}{\sum_{i\in\GGG} s_i}\sum_{i\in\GGG}s_i \left\|\Bar{\bx}_\GGG -{\bx}_i \right\|~.
\end{align*}

Taking the square of both sides and applying Jensen's inequality gives us: 
\begin{align*}
    \|\hat{\bx} - \Bar{\bx}_\GGG \|^2&
    \leq \left(1+\frac{\lambda}{1-2\lambda}\right)^2\frac{1}{\sum_{i\in\GGG} s_i}\sum_{i\in\GGG}s_i \left\|\Bar{\bx}_\GGG -{\bx}_i \right\|^2~.
\end{align*}
Taking the exception of both sides gives us the following:
\begin{align*}
    \E\|\hat{\bx} - \Bar{\bx}_\GGG \|^2&
    \leq \left(1+\frac{\lambda}{1-2\lambda}\right)^2\rho^2~.
\end{align*}
\end{proof}
\end{lemma}

\subsubsection{Weighted Coordinate-Wise Median (WeightedCWMed)}
The Weighted Coordinate-Wise Median (WeightedCWMed) is an aggregation technique that operates on a per-coordinate basis. For a given set of multi-dimensional data points $\{\bx_i\}_{i=1}^m$ and corresponding weights $\{s_i\}_{i=1}^m$, the WeightedCWMed is computed by independently finding the weighted median of each coordinate across all points. Formally, for the $k^{\rm{th}}$ dimension, the WeightedCWMed aggregator is defined as:
\[
[\textnormal{WeightedCWMed}]_k := \textnormal{WeightedMedian}(\{[\bx_i]_k\}_{i=1}^m; \{[s_i]_k\}_{i=1}^m)
\]
where $[\bx]_k$ is the $k^{\rm{th}}$ element of a vector $\bx$ and the WeightedMedian is defined as follows:
given the elements $\{\bx_1,\ldots,\bx_m\}$ of each dimension are sorted in ascending order and their corresponds weights $\{s_1,\ldots,s_m\}$, the weighted median is the element $\bx_{j^*}$, where $j^*$ is determined by the condition:
\[ j^* \in \arg \min_{j\in[m]}\left\{\sum_{i\in[j]}s_i > \frac{1}{2} \sum_{i\in[m]} s_i \right\} \]

If there exists a value $j$ such that:
\[ \sum_{i\in[j]}s_i = \frac{1}{2} \sum_{i\in[m]} s_i  \]
Then, the WeightedMedian is the average of the $j$-th and $(j+1)$-th elements:
\[ \textnormal{WeightedMedian} = \frac{\bx_j+\bx_{j+1}}{2} \]

Here, we extend the theoretical guarantee of the Coordinate-Wise Median (CWMed) to its weighted version, following the procedure in \cite{allouah2023fixing}.

\begin{lemma}
\label{lem:coordinate-wise}
    Let $A_\omega: \mathbb{R}^{d \times m} \rightarrow \mathbb{R}^d$ be a weighted coordinate-wise aggregation function. Given set of points $\{\bx_i\}_{i=1}^m$ and corresponding weights $\{s_i\}_{i=1}^m$, this function incorporates $d$ real-valued functions \( A_\omega^{1}, \ldots, A_\omega^{d} \), where each $[A_\omega(\{\bx_i\}_{i=1}^m; \{s_i\}_{i=1}^m)]_k = A_\omega^{k}(\{[\bx_i]_k\}_{i=1}^m; \{[s_i]_k\}_{i=1}^m)$. If for each $k\in[d]$, $A_\omega^{k}$ is $(c_\lambda, \lambda)$-weighted robust that satisfies:
    \begin{align*}
       \E\left| A_\omega^{k}(\{[\bx_i]_k\}_{i=1}^m; \{[s_i]_k\}_{i=1}^m) - [\Bar{\bx}_\GGG]_k \right|^2 & \leq   \frac{c_\lambda}{\sum_{i\in\GGG}s_i} \sum_{i\in\GGG}s_i \E\left| [\bx_i]_k - [\Bar{\bx}_\GGG]_k \right|^2~. 
    \end{align*}
    Then $A_\omega$ is $(c_\lambda, \lambda)$-weighted robust. 

    \begin{proof}
        Since $A_\omega$ is a coordinate-wise aggregator, it applies the same aggregation rule across each dimension. Therefore,   
        \begin{align*}
            \left\|A_\omega(\{\bx_i\}_{i=1}^m; \{s_i\}_{i=1}^m) - \Bar{\bx}_\GGG\right\|^2 = \sum_{k\in[d]}\left| A_\omega^{k}(\{[\bx_i]_k\}_{i=1}^m; \{[s_i]_k\}_{i=1}^m) - [\Bar{\bx}_\GGG]_k \right|^2~.
        \end{align*}
        Given that each $A_\omega^{k}$, for $k \in [d]$, is $(c_\lambda, \lambda)$-weighted robust  that satisfies: 
        \begin{align*}
            \E\left| A_\omega^{k}(\{[\bx_i]_k\}_{i=1}^m; \{[s_i]_k\}_{i=1}^m) - [\Bar{\bx}_\GGG]_k \right|^2 & \leq   \frac{c_\lambda}{\sum_{i\in\GGG}s_i} \sum_{i\in\GGG}s_i \E\left| [\bx_i]_k - [\Bar{\bx}_\GGG]_k \right|^2.
        \end{align*}
        We can express the overall aggregation function $A_\omega$ as follows:
        \begin{align*}
            \sum_{k\in[d]}\E\left| A_\omega^{k}(\{[\bx_i]_k\}_{i=1}^m; \{[s_i]_k\}_{i=1}^m) - [\Bar{\bx}_\GGG]_k \right|^2 & \leq   \sum_{k\in[d]}\frac{c_\lambda}{\sum_{i\in\GGG}s_i} \sum_{i\in\GGG}s_i \E\left| [\bx_i]_k - [\Bar{\bx}_\GGG]_k \right|^2 \\ & = \frac{c_\lambda}{\sum_{i\in\GGG}s_i} \sum_{i\in\GGG}s_i \E\sum_{k\in[d]}\left| [\bx_i]_k - [\Bar{\bx}_\GGG]_k \right|^2  \\ & = \frac{c_\lambda}{\sum_{i\in\GGG}s_i} \sum_{i\in\GGG}s_i \E\left\| \bx_i - \Bar{\bx}_\GGG \right\|^2 \\ & \leq c_\lambda \rho^2~,
        \end{align*}
where the first inequality is derived from the assumption stated in this lemma. The second aligns with the definition of $(c_\lambda, \lambda)$-weighted robust as detailed in Definition \ref{def2}. 
\end{proof}
\end{lemma}

\begin{lemma}
    Let ${\hat{\bx}}$ be a WeightedCWMed aggregator then ${\hat{\bx}}$ is $(c_\lambda, \lambda)$-weighted robust with $c_\lambda=\left(1+\frac{\lambda}{1-2\lambda}\right)^2$.

\begin{proof}
In the context of the $k^{\rm{th}}$ coordinate, $[\textnormal{WeightedCWMed}]_k$ functions equivalently to WeightedGM for a one-dimensional case. Consequently, each coordinate of the WeightedCWMed aggregator is $(c_\lambda, \lambda)$-weighted robust with $c_\lambda=\left(1+\frac{\lambda}{1-2\lambda}\right)^2$ as established in Lemma \ref{lem:WeightedGM}. Furthermore, since the WeightedCWMed functions on a coordinate-wise basis, it follows from Lemma \ref{lem:coordinate-wise} that the entire WeightedCWMed aggregator is $(c_\lambda, \lambda)$-weighted robust with $c_\lambda=\left(1+\frac{\lambda}{1-2\lambda}\right)^2$.
\end{proof}
\end{lemma}

\subsection{Proof of Lemma \ref{lem:CTMA}}
\label{app:ctma}
\begin{proof}[Proof of Lemma \ref{lem:CTMA}]

We denote $\by_i := \bx_i -\bx_0$, $\Sigma_\GGG:=\sum_{i\in\GGG}s_i$, $\Sigma_S:=\sum_{i\in S}s_i$, $\Sigma_\BBB:=\sum_{i\in \BBB}s_i$, and $\Sigma_m:=\sum_{i\in [m]}s_i$. Recall that $\Sigma_\GGG\geq(1-\lambda)\Sigma_m$, and $\Sigma_S=(1-\lambda)\Sigma_m$ (Alg. \ref{alg:CTMA}).
\begin{align*}
    \hat{\bx} - \bar{\bx}_\GGG &= \frac{1}{\Sigma_S} \sum_{i \in S}s_i\bx_i - \bar{\bx}_\GGG 
   \\& = \bx_0  - \bar{\bx}_\GGG + \frac{1}{\Sigma_S} \sum_{i \in S} s_i(\bx_i - \bx_0)\\ & = -\frac{1}{\Sigma_\GGG} \sum_{i \in \GGG} s_i(\bx_i - \bx_0) + \frac{1}{\Sigma_S} \sum_{i \in S} s_i(\bx_i - \bx_0)\\  
   &= -\frac{1}{\Sigma_\GGG}\sum_{i \in \GGG} s_i \by_i + \frac{1}{\Sigma_S} \sum_{i \in S} s_i \by_i \\ &= \left(\frac{1}{\Sigma_S}-\frac{1}{\Sigma_\GGG}\right)\sum_{i \in \GGG} s_i \by_i -\frac{1}{\Sigma_S} \sum_{i \in \GGG} s_i \by_i + \frac{1}{\Sigma_S}  \sum_{i \in S}s_i \by_i     \\
    &=  \left(\frac{1}{\Sigma_S}-\frac{1}{\Sigma_\GGG}\right)\sum_{i \in \GGG} s_i \by_i - \frac{1}{\Sigma_S} \sum_{i \in \GGG\backslash S} s_i\by_i + \frac{1}{\Sigma_S} \sum_{i \in S\backslash \GGG} s_i\by_i \\
    &=  \left(\frac{\Sigma_\GGG-\Sigma_S}{\Sigma_S\Sigma_\GGG}\right)\sum_{i \in \GGG} s_i \by_i - \frac{1}{\Sigma_S} \sum_{i \in \GGG\backslash S} s_i\by_i + \frac{1}{\Sigma_S} \sum_{i \in S\backslash \GGG} s_i\by_i.
\end{align*}
Taking the squared norm of both sides, we obtain:
\begin{align}
\label{eq:lambda}
    \|\hat{\bx} - \bar{\bx}_\GGG\|^2 &= \left\|\left(\frac{\Sigma_\GGG-\Sigma_S}{\Sigma_S\Sigma_\GGG}\right)\sum_{i \in \GGG} s_i \by_i - \frac{1}{\Sigma_S} \sum_{i \in \GGG\backslash S} s_i\by_i + \frac{1}{\Sigma_S} \sum_{i \in S\backslash \GGG} s_i\by_i \right\|^2 \nonumber \\
    &\leq 3\left\|\left(\frac{\Sigma_\GGG-\Sigma_S}{\Sigma_S\Sigma_\GGG}\right)\sum_{i \in \GGG} s_i \by_i\right\|^2 + 3\left\|\frac{1}{\Sigma_S} \sum_{i \in \GGG\backslash S} s_i\by_i\right\|^2 + 3\left\|\frac{1}{\Sigma_S} \sum_{i \in S\backslash \GGG} s_i\by_i\right\|^2 \nonumber \\
    &\leq 3\Sigma_\GGG \left(\frac{\Sigma_\GGG-\Sigma_S}{\Sigma_S\Sigma_\GGG}\right)^2 \sum_{i \in  \GGG} s_i\|\by_i\|^2 + \frac{3\sum_{i \in S\backslash \GGG}s_i}{\Sigma_S^2} \sum_{i \in S\backslash \GGG} s_i\|\by_i\|^2 + \frac{3\sum_{i \in \GGG\backslash S}s_i}{\Sigma_S^2} \sum_{i \in \GGG\backslash S}s_i \|\by_i\|^2~,
\end{align}
where the first inequality follows the inequality $\| \ba+\bb+\bc\|^2 \leq 3\|\ba\|^2+3\|\bb\|^2+3\|\bc\|^2$, $\forall \ba,\bb,\bc\in\real^2$ and the second follow Jensen’s inequality. 

Note that,
\begin{align}
\label{eq:bound1}
    \left(\frac{\Sigma_\GGG-\Sigma_S}{\Sigma_S\Sigma_\GGG}\right)^2&=\left(\frac{\Sigma_m-\Sigma_\BBB-(1-\lambda)\Sigma_m}{\Sigma_S\Sigma_\GGG}\right)^2 \nonumber
    \\&=\left(\frac{\lambda\Sigma_m-\Sigma_\BBB}{\Sigma_S\Sigma_\GGG}\right)^2 \nonumber
    \\&\leq\left(\frac{\lambda\Sigma_m}{(1-\lambda)\Sigma_m\Sigma_\GGG}\right)^2 \nonumber
    \\&\leq\left(\frac{2\lambda}{\Sigma_\GGG}\right)^2 \nonumber
    \\&<\frac{2\lambda}{\Sigma_\GGG^2}~,
\end{align}
where the first inequality holds because $\Sigma_\BBB\leq \lambda\Sigma_m$ and $\Sigma_S=(1-\lambda)\Sigma_m$. The second inequality follows from the fact that $1-\lambda\geq\nicefrac{1}{2}$, and the last since $\lambda<\nicefrac{1}{2}$. 

In addition, we have that,
\begin{align}
\label{eq:bound2}
    \frac{\sum_{i \in S\backslash \GGG}s_i}{\Sigma_S^2}&= \frac{\sum_{i \in S\cup \GGG}s_i-\sum_{i \in \GGG}s_i}{\Sigma_S^2}\nonumber
    \\&\leq\frac{\Sigma_m-(1-\lambda)\Sigma_m}{\Sigma_S^2} \nonumber
    \\&=\frac{\lambda\Sigma_m}{\Sigma_S^2} \nonumber
    \\&=\frac{\lambda\Sigma_m}{(1-\lambda)^2\Sigma_m^2} \nonumber
    \\&\leq\frac{4\lambda}{\Sigma_m} \nonumber
    \\&\leq \frac{4\lambda}{\Sigma_\GGG}~,
\end{align}
where the first inequality follows the facts that $S\cup\GGG\subseteq[m]$, $\{s_i\geq0\}_{i\in[m]}$ and $\Sigma_\GGG\geq(1-\lambda)\Sigma_m$. The second inequality is based on that $1-\lambda\geq\nicefrac{1}{2}$, and the last since $\Sigma_\GGG\leq\Sigma_m$. And in a similar way,  
\begin{align}
\label{eq:bound3}
    \frac{\sum_{i \in \GGG\backslash S}s_i}{\Sigma_S^2}\leq \frac{4\lambda}{\Sigma_\GGG}
\end{align} 

Applying Eq. \eqref{eq:bound1}, Eq. \eqref{eq:bound2} and Eq. \eqref{eq:bound3} into Eq. \eqref{eq:lambda}, gives us, 
\begin{align*}
\|\hat{\bx} - \bar{\bx}_\GGG\|^2 & \leq \frac{6\lambda}{\Sigma_\GGG}\sum_{i \in  \GGG} s_i\|\by_i\|^2 + \frac{12\lambda}{\Sigma_\GGG} \sum_{i \in S\backslash \GGG} s_i\|\by_i\|^2 + \frac{12\lambda}{\Sigma_\GGG} \sum_{i \in \GGG\backslash S}s_i \|\by_i\|^2
\\&\leq \frac{12\lambda}{\Sigma_\GGG} \sum_{i \in S\backslash \GGG} s_i \|\bx_i - \bx_0\|^2 + \frac{18\lambda}{\Sigma_\GGG} \sum_{i \in \GGG}s_i \|\bx_i - \bx_0\|^2~,
\end{align*}
where the latter holds since $\sum_{i \in \GGG\backslash S}s_i \|\by_i\|^2\leq\sum_{i \in \GGG}s_i \|\by_i\|^2$.

Next, we define:
\begin{align*}
    S^* &:= \bigcup_{i \in S} \{i\}_{j=1}^{s_i}, \quad 
    \GGG^* := \bigcup_{i \in \GGG} \{i\}_{j=1}^{s_i}.
\end{align*}

Note that $\sum_{i \in S\backslash \GGG} s_i \|\bx_i - \bx_0\|^2 = \sum_{i \in  {S}^*\backslash {\GGG}^*} \|\bx_i - \bx_0\|^2$. We'll show that there exists an injective function $\Phi:{S}^*\backslash{\GGG}^*\to {\GGG}^*\backslash{S}^*$ such that, \(\forall i\in S^*\setminus\GGG^* \), $\|\bx_{\Phi(i)}-\bx_0\|\geq\|\bx_i-\bx_0\|$ is satisfied. This follows from our selection of $S$, which consists of the closest elements $\{\bx_i\}_{i\in S}$ to  $\bx_0$ (see  Alg. \ref{alg:CTMA}), and from:
\begin{align*}
     |S^*\backslash \GGG^*| = \sum_{i \in S\backslash \GGG} s_i =  \sum_{i \in S} s_i + \sum_{i \in \G\backslash S} s_i - \sum_{i \in \GGG} s_i \leq \sum_{i \in \G\backslash S} s_i = |\GGG^*\backslash S^*|~,
\end{align*}
where the last inequality follows that $\sum_{i \in S} s_i - \sum_{i \in \GGG} s_i = (1-\lambda)\Sigma_m - \Sigma_\GGG \leq 0 $. 

Thus,
\begin{align*}
\|\hat{\bx} - \bar{\bx}_\GGG\|^2 &\leq \frac{12 \lambda}{\Sigma_\GGG} \sum_{i \in S\backslash \GGG} s_i \|\bx_i - \bx_0\|^2 + \frac{18 \lambda}{\Sigma_\GGG} \sum_{i \in \GGG} s_i \|\bx_i - \bx_0\|^2 \\
&= \frac{12 \lambda}{\Sigma_\GGG} \sum_{i \in S^*\backslash \GGG^*} \|\bx_{i} - \bx_0\|^2 + \frac{18 \lambda}{\Sigma_\GGG} \sum_{i \in \GGG} s_i\|\bx_i - \bx_0\|^2 \\
&\leq \frac{12 \lambda}{\Sigma_\GGG} \sum_{i \in S^*\backslash \GGG^*} \|\bx_{\Phi(i)} - \bx_0\|^2 + \frac{18 \lambda}{\Sigma_\GGG} \sum_{i \in \GGG} s_i\|\bx_i - \bx_0\|^2 \\
&\leq \frac{12 \lambda}{\Sigma_\GGG} \sum_{i \in \GGG^*} \|\bx_i - \bx_0\|^2 + \frac{18 \lambda}{\Sigma_\GGG} \sum_{i \in \GGG} s_i\|\bx_i - \bx_0\|^2 \\
&= \frac{12 \lambda}{\Sigma_\GGG} \sum_{i \in \GGG} s_i\|\bx_i - \bx_0\|^2 + \frac{18 \lambda}{\Sigma_\GGG} \sum_{i \in \GGG} s_i\|\bx_i - \bx_0\|^2 \\
&= \frac{30 \lambda}{\Sigma_\GGG} \sum_{i \in \GGG} s_i\|\bx_i - \bx_0\|^2 \\
&\leq \frac{60 \lambda}{\Sigma_\GGG} \sum_{i \in \GGG} s_i\|\bx_i - \bar{\bx}_\GGG\|^2 + \frac{60 \lambda}{\Sigma_\GGG} \sum_{i \in \GGG} s_i \|\bar{\bx}_\GGG - \bx_0\|^2~,
\end{align*}
where the second inequality follows from the definition of the injective function $\Phi$. The third inequality is justified by the fact that $\sum_{i \in \GGG^*\backslash S^*} \|\by_i\|^2\leq\sum_{i \in \GGG^*} \|\by_i\|^2$. Finally, the last inequality leverages the property $\|\ba+\bb\|^2\leq 2\|\ba\|^2+2\|\bb\|^2$, which holds $\forall\ba,\bb\in\real^d$.

Taking the expectations of both sides gives us the following:
\begin{align*}
\E\|\hat{\bx} - \bar{\bx}_\GGG\|^2 &\leq \frac{60\lambda}{\Sigma_\GGG} \sum_{i \in \GGG} s_i \E\|\bx_i - \Bar{\bx}_\GGG\|^2 + \frac{60 \lambda}{\Sigma_\GGG} \sum_{i \in \GGG} s_i \E\|\Bar{\bx}_\GGG - \bx_0\|^2 \\  & \leq {60 \lambda}\rho^2 + 60{\lambda} c_\lambda \rho^2  \\&=  60\lambda(1+c_\lambda) \rho^2~,
\end{align*}
where the last inequaility stems from Def.~\ref{def2}. 
\end{proof}

\section{Experiments}
\label{app:exp}
\subsection{Technical Details}
\paragraph{Datasets.} We simulated over the MNIST \citep{lecun2010mnist} and CIFAR-10 \citep{krizhevsky2014cifar} datasets. The datasets were accessed through \texttt{torchvision} (version 0.16.2).

\begin{itemize}
    \item \textbf{MNIST Dataset}. MNIST is a widely used benchmark dataset in the machine learning community, consisting of 70,000 grayscale images of handwritten digits (0-9) with a resolution of 28x28 pixels. The dataset is split into 60,000 training images and 10,000 test images.
    
    \item \textbf{CIFAR-10 Dataset}. CIFAR-10 is a widely recognized benchmark dataset in the machine learning community, containing 60,000 color images categorized into 10 different classes. Each image has a resolution of 32x32 pixels and represents objects such as airplanes, automobiles, birds, cats, and more. The dataset is evenly split into 50,000 training images and 10,000 test images.
\end{itemize}
\begin{table}[h]
    \centering
    \renewcommand{\arraystretch}{1.3} 
    \begin{tabularx}{\textwidth}{>{\raggedright\arraybackslash}X >{\centering\arraybackslash}X >{\centering\arraybackslash}X}
        \toprule
        \textbf{Parameter} & \textbf{MNIST} & \textbf{CIFAR-10} \\
        \midrule
        Model Architecture & 
        \small Conv(1,20,5), ReLU, MaxPool(2x2), Conv(20,50,5), ReLU, MaxPool(2x2), FC(800$\rightarrow$50), BatchNorm, ReLU, FC(50$\rightarrow$10) &
        \small Conv(3,20,5), ReLU, MaxPool(2x2), Conv(20,50,5), ReLU, MaxPool(2x2), FC(1250$\rightarrow$50), BatchNorm, ReLU, FC(50$\rightarrow$10) \\
        \addlinespace
        Learning Rate & 0.01 & 0.01 \\
        \addlinespace
        Batch Size & 16 & 8 \\
        \addlinespace
        Data Processing \& Augmentation & \small Normalize(mean=(0.1307), std=(0.3081)) &
        \small RandomCrop(size=32, padding=2), RandomHorizontalFlip(p=0.5), Normalize(mean=(0.4914, 0.4822, 0.4465), std=(0.2023, 0.1994, 0.2010)) \\
        \bottomrule
    \end{tabularx}
    \caption{Experimental Setup for MNIST and CIFAR-10}
    \label{tab:experiment-setup}
\end{table}

\paragraph{Imbalanced Arrival Scenarios.} We simulated two types of imbalanced arrival scenarios:
\begin{itemize}
    \item \textbf{Proportional Arrival Probability}: The probability of arrival for the \(i\)-th worker in the honest group was set proportionally to \(i/\sum_{j \in \mathcal{G}} j\), ensuring that workers with higher indices had a higher chance of arriving. The same distribution method was applied to Byzantine workers.
    \item \textbf{Squared ID Arrival Probability}: In a more skewed scenario, the arrival probability was proportional to the square of the worker’s ID, i.e., \(i^2/\sum_{j \in \mathcal{G}} j^2\). This setup further accentuated the imbalance by favoring workers with larger IDs.
\end{itemize}

For simplicity, Byzantine workers were introduced after a fixed number of iterations, controlled by a parameter \(\lambda\). However, it is worth noting that when Byzantine iterations are concentrated, they can cause significant performance degradation. Such patterns can lead to increased delays for honest updates, ultimately affecting the overall convergence of the algorithm.

\paragraph{Optimization Setup.} We optimized the cross-entropy loss across all experiments. For comparisons, we configured \(\mu^2\)-SGD with fixed parameters \(\gamma = 0.1\) and \(\beta = 0.25\). This was tested against Standard SGD, and Momentum-based SGD, where the momentum parameter was set to \(\beta = 0.9\) as recommended by \cite{karimireddy2021learning}.

\paragraph{Attack Simulations.} We simulated four types of attacks to evaluate the robustness of our approach: 
\begin{enumerate}
    \item \textbf{Label Flipping}  \citep{allen2020byzantine}: The labels of the data were flipped to incorrect values, by subtracting the original labels from 9.
    \item \textbf{Sign Flipping}  \citep{allen2020byzantine}: The signs of the workers' output were flipped.
    \item \textbf{Little} \citep{baruch2019little}: Adapted from the synchronous case.  It computes the maximum allowable deviation \( z_{\text{max}} \) based on iterations count rather than the number of workers. Then, it perturbs the honest updates by subtracting the product of the weighted standard deviation and \( z_{\text{max}} \) from the weighted mean of the honest updates.
      \[
     \text{Byzantine\_update} = \text{weighted\_mean}(\text{honest\_momentums}) - \text{weighted\_std}(\text{honest\_momentums}) \cdot z_{\text{max} }.
     \]
    \item \textbf{Empire} \citep{xie2020fall}: Adapted from the synchronous case.  This attack scales the weighted mean of the honest momentums by a factor \(\epsilon\) in the negative direction, 
    \[
    \text{Byzantine\_update} = -\epsilon \cdot \text{weighted\_mean}(\text{honest\_momentums}).
    \]
\end{enumerate}
In the two latter attacks, the mean and standard deviation are calculated coordinate-wise with respect to weights, setting \(\epsilon = 0.1\).

\paragraph{AnyTime Update Formulation.} 
 Regarding the AnyTime update, defined as \(\bx_t := \frac{\alpha_t \bw_t + \alpha_{1:t-1} \bx_{t-1}}{\alpha_{1:t}}\), we employed a momentum-based formulation that equivalent to the standard AnyTime update. Specifically, we updated the model parameters according to the formula:
$$
\bx_t = \gamma_t \bw_t + (1 - \gamma_t) \bx_{t-1}
$$
where \(\gamma_t\) is difined as \(\gamma_t := \frac{\alpha_t}{\alpha_{1:t}}\). By setting \(\alpha_t = C \alpha_{1:t-1}\) with \(C>0\) being a constant, \ we derived that \(\gamma_t = \frac{C}{C + 1}\) and remains consistent across all time steps \(t \geq 1\).

For more details, please visit our GitHub repository.\footnote{\url{https://github.com/dahan198/asynchronous-fault-tolerant-ml}}

\clearpage
\subsection{Experimental Results on MNIST}
\label{app:results}

\begin{figure}[ht]
\centering
\includegraphics[width=1\linewidth]{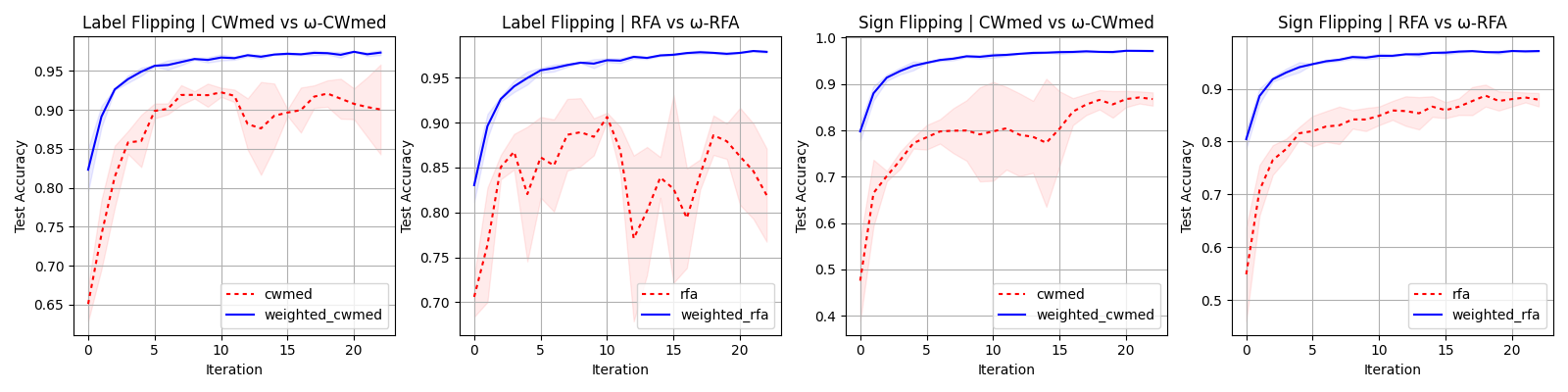}
\caption{\small \textbf{MNIST}. \textbf{Test Accuracy of Weighted vs. Non-Weighted Robust Aggregators}. This scenario involves 17 workers, including 8 Byzantine workers. The arrival probabilities of workers are proportional to the square of their IDs, with $\lambda = 0.4$. We employed $\mu^2$-SGD in this setup. {Left}: \textit{label flipping} attack. {Right}: \textit{sign flipping} attack.}
\label{fig:weighted}
\end{figure}

\begin{figure}[ht]
\centering
\includegraphics[width=1\linewidth]{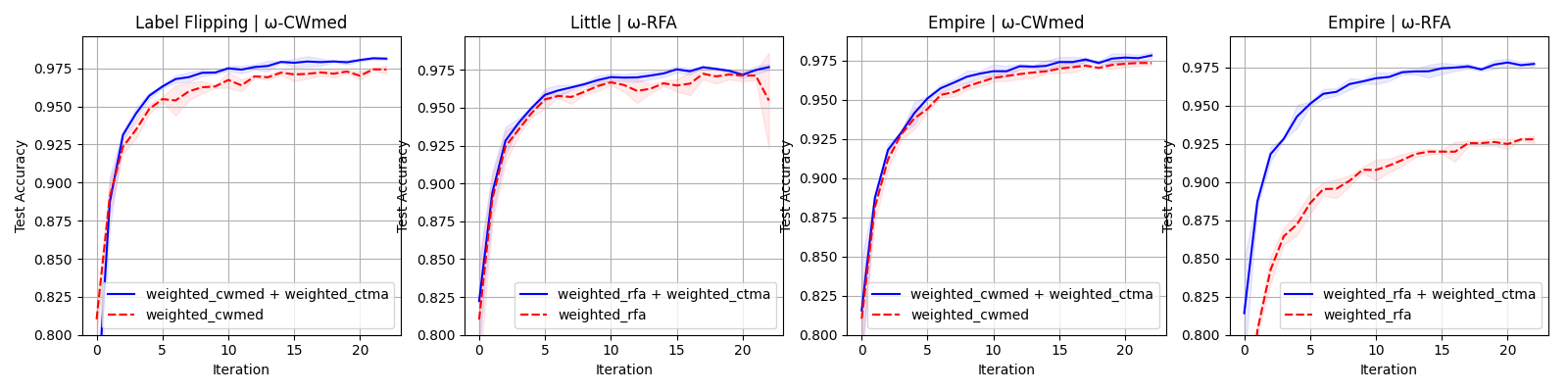}
\caption{\small \textbf{MNIST}. \textbf{Test Accuracy Comparison of Weighted Robust Aggregators With and Without $\omega$-CTMA}. This scenario involves 9 workers, with a very fast Byzantine worker, and workers' arrival probabilities proportional to their IDs. We used the $\mu^2$-SGD in this scenario.  On the \textit{left} we have \textit{label flipping}  ($\lambda=0.4$) and \textit{little} ($\lambda=0.3$) attacks. On the \textit{right} we have an \textit{empire} ($\lambda=0.4$) attack.}
\label{fig:ctma}
\end{figure}

\begin{figure}[ht]
\centering
\includegraphics[width=1\linewidth]{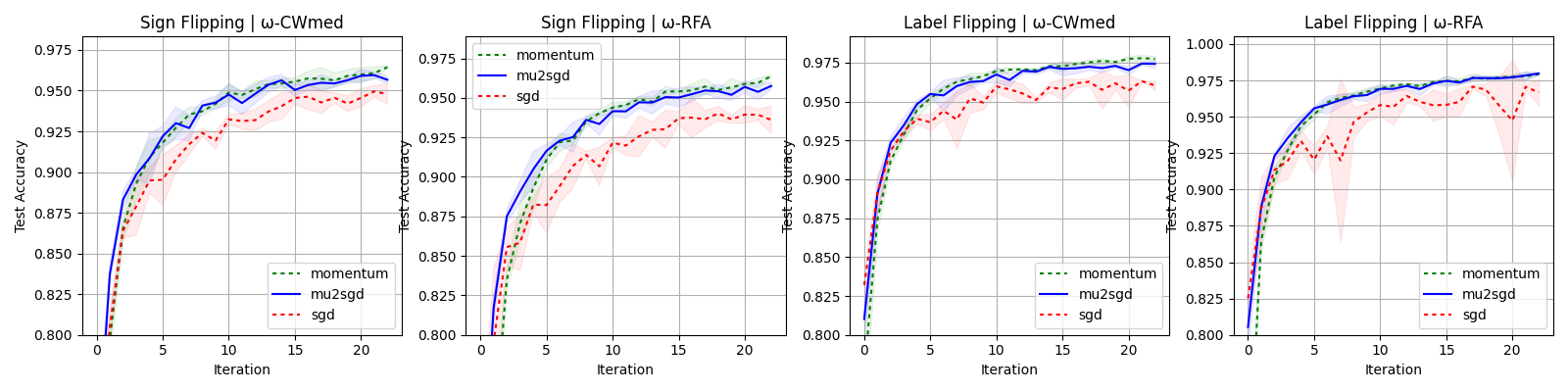}
\caption{\small \textbf{MNIST}. \textbf{Test Accuracy Comparison Among Different Optimizers}. This scenario involves 9 workers, with a very fast Byzantine worker, $\lambda=0.4$, and workers' arrival probabilities proportional to their IDs. {Left}: \textit{sign flipping}. {Right}: 
 \textit{label flipping}.}
\label{fig:optimizer}
\end{figure}

\clearpage

\end{document}